\documentclass{llncs}
\usepackage{times}
\usepackage{xcolor}
\usepackage{amsmath}
\usepackage{amsfonts}
\usepackage{graphicx}
\usepackage{subfig}

\usepackage[colorlinks=true,linkcolor=blue]{hyperref}

\newcommand{\norm}[1]{\|#1\|}
\newcommand{\xplus}{\bs{X}_{t+1}}
\newcommand{\xf}[1][t]{\left[ \bs{X}_{#1} \right]_1}
\newcommand{\x}[1][t]{\bs{X}_{#1}}
\newcommand{\z}[1][t]{Z_{#1}}
\newcommand{\suite}[1]{({#1}_t)_{t \in \N}}
\newcommand{\suitet}[1]{(#1)_{t \in \N}}

\newcommand{\sig}[1][t]{\sigma_{#1}}
\newcommand{\sigplus}{\sigma_{t+1}}

\newcommand{\yi}{\bs{Y}_{t, i}}
\newcommand{\y}[1]{\bs{Y}_{t, #1}}
\newcommand{\pplus}{\bs{p}_{t+1}}
\newcommand{\p}[1][t]{\bs{p}_{#1}}

\newcommand{\pdim}[1][i]{\left[\pplus\right]_{#1}}
\newcommand{\inv}[1]{\frac{1}{#1}}

\newcommand{\zchain}{\mathcal{Z}}
\newcommand{\randomstep}[1][i]{\bs{\xi}_{t,#1}}
\newcommand{\chosenstep}[1][t]{\bs{\xi}_{#1}^\star}
\newcommand{\chosenstepdistribution}{\bs{\xi}}
\newcommand{\stepsizechangelong}[1][\bs{p}_{t+1}]{\exp\left( \frac{c}{2d_\sigma}\left( \frac{\norm{#1}^2}{n} - 1 \right)\right)}
\newcommand{\stepsizechange}[1][t]{\eta_{#1}^\star}
\newcommand{\hsp}{\hspace{1mm}}
\newcommand{\firstdim}[1]{\left[#1\right]_1}
\newcommand{\fchosenstep}[1][t]{\firstdim{\chosenstep[#1]}}

\newcommand{\pchain}{\bs{\mathcal{P}}}
\newcommand{\fpchain}{[\bs{\mathcal{P}}]_1}

\newcommand{\E}{\mathbb{E}}
\newcommand{\R}{\mathbb{R}}
\newcommand{\NL}{\bs{\mathcal{N}}(\bs{0}, {I_n)}}
\newcommand{\Nlambda}[1][1]{\mathcal{N}_{#1:\lambda}}
\newcommand{\Nmin}[1][\lambda]{\mathcal{N}_{1:{#1}}}
\newcommand{\N}{\mathbb{N}}
\newcommand{\NNN}{\mathcal{N}}

\newcommand{\minchii}{\chi_{1:\lambda}^{\{i\}}}
\newcommand{\minchiisort}{\chi_{1 \lambda}^{[i]}}

\newcommand{\bs}[1]{\boldsymbol{#1}}

\DeclareMathOperator{\argmin}{argmin}
\DeclareMathOperator{\Var}{Var}

\begin{document}

\title{Cumulative Step-size Adaptation on Linear Functions: Technical Report}
\author{Alexandre Chotard\inst{1}, Anne Auger\inst{1} \and
Nikolaus Hansen\inst{1}}
\authorrunning{Alexandre Chotard et al.} 
\tocauthor{Alexandre Chotard, Anne Auger and
Nikolaus Hansen}
\institute{TAO team, INRIA Saclay-Ile-de-France, LRI, Paris-Sud University,
France\\
\email{firstname.lastname@lri.fr}}

\maketitle       

\begin{abstract}
The CSA-ES is an Evolution Strategy with Cumulative Step size Adaptation, where the step size is adapted measuring the length of a so-called cumulative path. The cumulative path is a combination of the previous steps realized by the algorithm, where the importance of each step decreases with time.
This article studies the CSA-ES on composites of strictly increasing functions with affine linear functions through the investigation of its underlying Markov chains. Rigorous results on the change and the variation of the step size are derived with and without cumulation. The step-size diverges geometrically fast in most cases.  Furthermore, the influence of the cumulation parameter is studied.
\end{abstract}

\keywords{CSA, cumulative path, evolution path, evolution strategies, step-size adaptation}

\section{Introduction}
\newcommand{\mc}[2]{\newcommand{#1}{\ensuremath{#2}}}
\mc{\X}{\bs{X}}

Evolution strategies (ESs) are continuous stochastic optimization algorithms searching for the minimum of a real valued function $f:\R^n\to\R$. In the ($1,\lambda$)-ES, in each iteration, $\lambda$ new children are generated from a single parent point $\X\in\R^n$ by adding a random Gaussian vector to the parent,

\[
	\bs{X} \in \R^n \mapsto \bs{X} + \sigma \bs{\mathcal{N}} (\bs{0}, \bs{C}) 
	\enspace.
\]
Here, $\sigma\in \R_{+}^{*}$ is called step-size and $\bs{C}$ is a covariance matrix. The best of the $\lambda$ children, i.e.\ the one with the lowest $f$-value, becomes the parent of the next iteration. 
To achieve reasonably fast convergence, step size and covariance matrix have to be adapted throughout the iterations of the algorithm. In this paper, $\bs{C}$ is the identity and we investigate the so-called Cumulative Step-size Adaptation (CSA), which is used to adapt the step-size in the Covariance Matrix Adaptation Evolution Strategy (CMA-ES) \cite{ostermeier1994nonlocal,cmaesbirth}. In CSA, a cumulative path is introduced, which is a combination of all steps the algorithm has made, where the importance of a step decreases exponentially with time. Arnold and Beyer studied the behavior of CSA on sphere, cigar and ridge functions \cite{arnold2004performance,arnold2010behaviour,arnold2006cumulative,arnold2008evolution} and on dynamical optimization problems where the optimum moves randomly \cite{arnold2002random} or linearly \cite{arnold2006optimum}. Arnold also studied the behaviour of a ($1,\lambda$)-ES on linear functions with linear constraint \cite{arnold2011behaviour}. 

In this paper, we study the behaviour of the $(1,\lambda)$-CSA-ES on composites of strictly increasing functions with affine linear functions, e.g.\ $f:\vec x\mapsto \exp(x_2 - 2)$. Because the CSA-ES is invariant under translation, under change of an orthonormal basis (rotation and reflection), and under strictly increasing transformations of the $f$-value, we investigate, w.l.o.g., $f: \vec x \mapsto x_1$. Linear functions model the situation when the current parent is far (here infinitely far) from the optimum of a smooth function. To be far from the optimum means that the distance to the optimum is large, \emph{relative to the step-size $\sigma$}. This situation is undesirable and threatens premature convergence. 
The situation should be handled well, by increasing step widths, by any search algorithm (and is not handled well by the $(1,2)$-$\sigma$SA-ES \cite{hansen2006ecj}). Solving linear functions is also very useful to prove convergence independently of the initial state on more general function classes.

In Section~\ref{sec:CSA} we introduce the $(1, \lambda)$-CSA-ES, and some of its characteristics on linear functions. In Sections~\ref{sec:withoutcumulation} and \ref{sec:withcumulation} we study $\ln(\sigma_t)$ without and with cumulation, respectively. Section \ref{sec:var} presents an analysis of the variance of the logarithm of the step-size and in Section~\ref{sec:conclusion} we summarize our results.

\mc{\xx}{\bs{x}}
\mc{\NN}{\mathcal{N}(0,1)}
\paragraph*{Notations}
In this paper, we denote $t$ the iteration or time index, $n$ the search space dimension, \NN\ a standard normal distribution, i.e. a normal distribution with mean zero and standard deviation 1. The multivariate normal distribution with mean vector zero and covariance matrix identity will be denoted $\NL$, the $i^{\rm th}$ order statistic of $\lambda$ standard normal distributions $\Nlambda[i]$, and $\Psi_{i:\lambda}$ its distribution. If $\bs{x} = \left(x_1, \cdots, x_n \right) \in \R^n$ is a vector, then $\left[x\right]_i$ will be its value on the $i^{th}$ dimension, that is $\left[x\right]_i = x_i$. A random variable $\bs{X}$ distributed according to a law $\mathcal{L}$ will be denoted $\bs{X} \sim \mathcal{L}$. If $A$ is a subset of $\mathcal{X}$, we will denote $A^c$ its complement in $\mathcal{X}$.

\section{The $(1, \lambda)$-CSA-ES} \label{sec:CSA}

We denote with $\x$ the parent at the $t^{th}$ iteration.
From the parent point $\x$, $\lambda$ children are generated: $\yi = \x + \sig[t] \randomstep$ with $i \in [[1, \lambda]]$, and $\randomstep \sim \NL, \hsp (\randomstep)_{i \in [[1, \lambda]]} \hsp$ i.i.d.  Due to the $(1,\lambda)$ selection scheme, from these children, the one minimizing the function $f$ is selected: $\xplus = \argmin \{f(\bs{Y}), \bs{Y} \in \{{\y{1}, ..., \y{\lambda}\}}\}$.
This latter equation implicitly defines the random variable $\chosenstep$ as
\begin{equation}
\xplus = \x + \sig[t] \chosenstep \label{eq:selection2}
\enspace.
\end{equation}
In order to adapt the step-size, the cumulative path is defined as
\begin{equation}
\pplus = (1-c)\p + \sqrt{c(2-c)}\, \chosenstep \label{eq:chemin}
\end{equation}
with $0<c\leq 1$. The constant $1/c$ represents the life span of the information contained in $\p$, as after $1/c$ generations $\p$ is multiplied by a factor that approaches $1/e \approx 0.37$ for $c\to0$ from below (indeed $(1-c)^{1/c} \leq \exp(-1)$). The typical value for $c$ is between $1/\sqrt{n}$ and $1/n$. We will consider that $\bs{p}_0 \sim \NL$ as it makes the algorithm easier to analyze.

The normalization constant $\sqrt{c(2-c)}$ in front of $\chosenstep$ in Eq.~\eqref{eq:chemin} is  chosen so that under random selection and if $\p$ is distributed according to $\NL$ then also $\pplus$ follows $\NL$. Hence the length of the path can be compared to the expected length of $\| \NL \|$ representing the expected length under random selection. 

The step-size update rule increases the step-size if the length of the path is larger than the length under random selection and decreases it if the length is shorter than under random selection: 
$$
\sigplus = \sig[t] \exp\left({\frac{c}{d_{\sigma}}\left(\frac{\|\pplus\|}{E(\|\NL\|)} - 1\right)}\right) 
$$
where the damping parameter $d_{\sigma}$ determines how much the step-size can change and is set to $d_{\sigma}=1$. 
A simplification of the update considers the squared length of the path~\cite{arnold2002random}:
\begin{equation}
\sigplus = \sig[t] \stepsizechangelong. \label{eq:pas}
\end{equation}
This rule is easier to analyse and we will use it throughout the paper. We will denote $\stepsizechange$ the random variable for the step-size change, i.e. $\stepsizechange = \exp(c/(2d_\sigma )(\|\pplus\|^2/n - 1))$, and for $\bs{u} \in \R^n$, $\eta^\star(\bs{u}) = \exp(c/(2d_\sigma )(\|u\|^2/n - 1))$.

\paragraph{Preliminary results on linear functions.}

Selection on the linear function, $f(\xx)=[\xx]_{1}$, is determined by  $\firstdim{\x} + \sig\firstdim{\chosenstep}  \leq  \firstdim{\x} + \sig\firstdim{\randomstep} $ for all $i$ which is equivalent to $ \firstdim{\chosenstep}  \leq  \firstdim{\randomstep} $ for all $i$ where by definition $\firstdim{\randomstep}$ is distributed according to $\NN$. Therefore the first coordinate of the selected step is distributed according to $\Nmin$ and all others coordinates are distributed according to $\NN$, i.e. selection does not bias the distribution along the coordinates $2,\ldots,n$. Overall we have the following result.
\begin{lemma}\label{lem:selectedstep}
On the linear function $f(\xx)=x_{1}$, the selected steps $(\chosenstep)_{t\in \N}$ of the $(1,\lambda)$-ES are i.i.d. and distributed according to the vector $\chosenstepdistribution:=(\Nmin,\NNN_{2},\ldots,\NNN_{n})$ where $\NNN_{i} \sim \NN$ for $i \geq 2$.
\end{lemma}

Because the selected steps $\chosenstep$ are i.i.d.\ the path defined in Eq.~\ref{eq:chemin} is an autonomous Markov chain, that we will denote $\pchain = (\p)_{t\in \N}$. Note that if the distribution of the selected step depended on $(\x,\sigma_{t})$ as it is generally the case on non-linear functions, then the path alone would not be a Markov Chain, however $(\x,\sigma_{t},\p)$ would be an autonomous Markov Chain. In order to study whether the $(1,\lambda)$-CSA-ES diverges geometrically, we investigate the log of the step-size change, whose formula can be immediately deduced from Eq.~\ref{eq:pas}:
\begin{equation}\label{eq:stepsizechange}
 \ln\left( \frac{\sigma_{t+1}}{\sig} \right)   =  \frac{c}{2d_\sigma}  
\left( \frac{\norm{\pplus}^2}{n} - 1 \right)  
\end{equation}
By summing up this equation from $0$ to $t-1$ we obtain
\begin{equation} \label{eq:sigcumul}
\inv{t}\ln\left(\frac{\sigma_{t}}{\sigma_0} \right) = \frac{c}{2d_\sigma}  \left( \inv{t}\sum_{k=1}^{t} \frac{\norm{\p[k]}^2}{n} - 1 \right) \enspace.
\end{equation}
We are interested to know whether $\inv{t}\ln ({\sigma_{t}/ \sigma_0} )$ converges to a constant. In case this constant is positive this will prove that the $(1,\lambda)$-CSA-ES diverges geometrically. We recognize thanks to \eqref{eq:sigcumul} that this quantity is equal to the sum of $t$ terms divided by $t$ that suggests the use of the law of large numbers to prove convergence of \eqref{eq:sigcumul}. We will start by investigating the case without cumulation $c=1$ (Section~\ref{sec:withoutcumulation}) and then the case with cumulation (Section~\ref{sec:withcumulation}).

\section{Divergence rate of $(1, \lambda)$-CSA-ES without cumulation} 

\label{sec:withoutcumulation}
In this section we study the $(1,\lambda)$-CSA-ES without cumulation, i.e. $c=1$. In this case, the path always equals to the selected step, i.e.\ for all $t$, we have $\pplus = \chosenstep$. We have proven in Lemma~\ref{lem:selectedstep} that $\chosenstep$ are i.i.d.\ according to $\chosenstepdistribution$. This allows us to use the standard law of large numbers to find the limit of $\frac1t \ln(\sigma_{t}/\sigma_{0})$ as well as compute the expected log-step-size change.

\mc{\Deltasigma}{\Delta_\sigma}
\begin{proposition} \label{pr:sigrate}
Let $\Deltasigma := \frac{1}{2d_{\sigma}n} \left( \E \left( \Nlambda^2 \right) - 1 \right)$. On linear functions, the $(1,\lambda)$-CSA-ES without cumulation satisfies (i) almost surely
$ \lim_{t \to \infty}	\inv{t} \ln\left({\sig}/{\sig[0]}\right) = \Deltasigma 
$,
and (ii) for all $t\in\N$, 
$\E (\ln({\sigma_{t+1}}/{\sig})) = \Deltasigma
$.
\end{proposition}
\begin{proof} 
We have identified in Lemma~\ref{lem:selectedstep} that the first coordinate of $\chosenstep$ is distributed according to $\Nmin$ and the other coordinates according to $\mathcal{N}(0,1)$, hence $\E\left(\|\chosenstep\|^2\right)=\E\left( {\firstdim{\chosenstep}}^2  \right) + \sum_{i=2}^n \E\left( \left[\chosenstep\right]_i^2  \right) = \E\left( \Nlambda^2 \right) + n - 1  $. Therefore $\E\left(\|\chosenstep\|^2\right)/n - 1 = (\E \left( \Nlambda^2 \right) - 1)/n $.
By applying this to Eq.~\eqref{eq:stepsizechange}, we deduce that $\E( \ln(\sigma_{t+1}/\sigma_t) = 1/(2d_\sigma n)(\E(\Nmin^2) - 1)$. Furthermore, as $\E ( \Nlambda^2 ) \leq \E ((\lambda\NN)^2) = \lambda^2 < \infty$, we have $\E(\|\chosenstep\|^2) < \infty$. The sequence $(\|\chosenstep\|^2)_{t \in \N}$ being i.i.d according to Lemma~\ref{lem:selectedstep}, and being integrable as we just showed, we can apply the strong law of large numbers on Eq.~\eqref{eq:sigcumul}. We obtain
\begin{align*}
\inv{t}\ln \left({\sig \over \sigma_0} \right) &= \frac{1}{2d_\sigma}\left( \inv{t}\sum_{k=0}^{t-1}{\|\chosenstep[k]\|^2 \over n}  - 1  \right)\\
&\overset{a.s.}{\underset{t \rightarrow \infty}{\longrightarrow}} \frac{1}{2d_\sigma} \left( \frac{\E \left( \|\chosenstep[\cdot]\|^2 \right)}{n}   - 1   \right) = \frac{1}{2d_\sigma n} \left(  \E \left( \Nlambda^2 \right) - 1   \right) 
\end{align*}
~\\[-1.8\baselineskip]
\qed\end{proof}
%

The proposition reveals that the sign of $\left( \E \left( \Nlambda^2 \right) - 1 \right)$ determines whether the step-size diverges to infinity. In the following, we show that $\E \left( \Nlambda^2 \right)$ increases in $\lambda$ for $\lambda \geq 2$ and that the $(1,\lambda)$-ES diverges for $\lambda \geq 3$. For $\lambda=1$ and $\lambda=2$, the step-size follows a random walk on the log-scale. To prove this we need the following lemma:
\begin{lemma}[\cite{expectedlambda}] \label{lm:niko_lambda_equation}
Let $g$ be a real valued function on $\R$. For $\lambda \geq 2$,
\begin{equation} \label{eq:niko_lambda_equation}
\left(\lambda+1\right)\E\left(g\left(\mathcal{N}_{1:\lambda}\right)\right) = \E\left(g\left(\mathcal{N}_{2:\lambda+1}\right)\right) + \lambda\E\left(g\left(\mathcal{N}_{1:\lambda+1}\right)\right) \enspace.\end{equation}
\end{lemma}
\begin{proof} of Lemma \ref{lm:niko_lambda_equation}

This method can be found with more details in \cite{expectedlambda}.
	
	Let $\chi_i = g\left(\xi_i\right)$, and $\chi_{i:\lambda} = g\left(\xi_{i:\lambda}\right)$. Note that in general $\chi_{1:\lambda} \neq \underset{i \in [[1,\lambda]]}{\min} \chi_i$. The sorting is made on $(\xi_i)$, not on $(\chi_i)$.
	
	We will also note $\chi_{i:\lambda}^{\{j\}}$ the $i^th$ order statistic after that the variable $\chi_j$ has been taken away. $\chi_{i:\lambda}^{[j]}$ will be $i^th$ order statistic after $\chi_{j:\lambda}$ has been taken away : if $i \neq 1$ then we have $\chi_{1:\lambda}^{[i]} = \chi_{1:\lambda}$, and for $i = 1$ $\chi_{1:\lambda}^{[i]} = \chi_{2:\lambda}$.
	
	Then we have $\E\left(\chi_{1:\lambda}^{\{i\}}\right) = \chi_{1:\lambda-1}$,
	
	And $\sum_{i=1}^\lambda \chi_{1:\lambda}^{\{i\}} = \sum_{i=1}^\lambda \chi_{1:\lambda}^{[i]}$ $(2)$.
	
	From the first equation we deduce that $\lambda\E\left(\chi_{1:\lambda-1}\right) = \lambda \E(\minchii) = \sum_{i=1}^\lambda \E(\minchii) =  \E(\sum_{i=1}^\lambda \minchii)$.
	
	With the second equation, we get that $\E(\sum_{i=1}^\lambda \minchii) = \E(\sum_{i=1}^\lambda \minchiisort) = \E(\chi_{2:\lambda}) + (\lambda-1)\E(\chi_{1:\lambda})$.
	
	By combining both, we get the final equation:
	
		$$(\lambda - 1)(\E(\chi_{1:\lambda}) - \E(\chi_{1:\lambda-1})) = \E(\chi_{1:\lambda-1}) - \E(\chi_{2:\lambda})$$
\qed\end{proof}
We are now ready to prove the following result.
\begin{lemma} \label{lm:increasing_lambda_expectation}
	Let $(\mathcal{N}_i)_{i \in [[1, \lambda]]}$ be independent random variables, distributed according to $\NN$, and $\mathcal{N}_{i:\lambda}$ the $i^{th}$ order statistic of $(\mathcal{N}_i)_{i \in [[1, \lambda]]}$. Then $\E \left( \Nmin[1]^2 \right)=\E \left( \Nmin[2]^2 \right)= 1$. In addition, for all $\lambda \ge 2$, $\E \left( \Nmin[\lambda+1]^2 \right) > \E \left( \Nmin^2 \right)$. 
\end{lemma}

\begin{proof} of Lemma \ref{lm:increasing_lambda_expectation}
The strict monotony of $\E(\Nmin^2)$ in $\lambda$ from the previous proposition is equivalent to show that $\E(\mathcal{N}_{1:\lambda}^2) > \E\left(\mathcal{N}_{2:\lambda}^2 \right)$ for $\lambda \geq 3$. Indeed $\E(\Nmin^2)- \E(\Nmin[\lambda-1]^2) = \E(\mathcal{N}_{1:\lambda}^2 - \mathcal{N}_{2:\lambda}^2 )/ \lambda$ which follows from  Lemma~\ref{lm:niko_lambda_equation} taking $g$ as the square function. 

Let $E_1 = \lbrace \omega \in \Omega | \Nlambda^2(\omega) < \Nlambda[2]^2(\omega) \rbrace$, where $\Omega = \R^\lambda$ and $P(\omega) = \exp(- \|\omega\|^2/2)/\sqrt{2\pi}$. For $\omega \in \Omega$, let us note $\omega_{i:\lambda}$ the $i^{\rm th}$ order statistic of $([\omega]_j)_{j \in [[1, \lambda]]}$. Let $g$ be a function that maps $\omega \in \Omega$ to $\tilde{\omega} \in \Omega$, where $\tilde{\omega}_{1:\lambda} = -\omega_{2:\lambda}$, $\tilde{\omega}_{2:\lambda} = \omega_{1:\lambda}$ and for $i \geq 3$, $\tilde{\omega}_{i:\lambda} = \omega_{i:\lambda}$. The function $g$ is bijective between $E_1$ and its image by $g$, $E_2$. Let us note that for $\omega \in E_1$, $\Nlambda[2]^2(\omega) - \Nlambda^2(\omega) = \Nlambda^2(g(\omega)) - \Nlambda[2]^2(g(\omega))$, and $P(\omega) = P(g(\omega))$ since the standard normal distribution is symmetric. That is $\int_{E_1} (\Nlambda[2]^2(\omega)-\Nlambda^2(\omega)) P(\omega)d\omega = \int_{E_2} (\Nlambda^2(\tilde{\omega})-\Nlambda[2]^2(\tilde{\omega})) P(\tilde{\omega})d\tilde{\omega}$ by a change of variables $\tilde{omega} = g(\omega)$. As according to the definition of $E_1$, for all $\omega \in \Omega \backslash E_1$ $\Nlambda^2(\omega) \geq \Nlambda[2]^2(\omega)$, and that $E_1$ is properly counterweighted by $E_2$ in the expected value of $\Nlambda^2 - \Nlambda[2]^2$, we do have $\E(\Nlambda^2) \geq \E(\Nlambda[2]^2)$ for all $\lambda \geq 2$.

For $\lambda \geq 3$, let $E_3 = \lbrace \omega \in \Omega | \omega_{3:\lambda} \in ]-|\omega_{1:\lambda}|, |\omega_{1:\lambda}|[ ~~{\rm and }~~ \omega_{1:\lambda} < \omega_{2:\lambda} \rbrace$. Then, for $\omega \in E_3$ we also have $\omega_{2:\lambda} \in ]-|\omega_{1:\lambda}|, |\omega_{1:\lambda}|[$, so $\Nlambda^2(\omega) > \Nlambda[2]^2(\omega)$ which means $\omega \notin E_1$, or $E_1 \cap E_3 = \emptyset$. For $\omega \in E_1$, as $\omega_{1:\lambda}^2 < \omega_{2:\lambda}^2$ and $\omega_{1:\lambda} < \omega{2:\lambda}$, $\omega_{2:\lambda} > 0$, so $\omega_{3:\lambda} \notin [-\omega{2:\lambda}, \omega{2:\lambda}]$. Hence, as $g(\omega)_{3:\lambda} = \omega{3:\lambda}$ and $[-\omega{2:\lambda}, \omega{2:\lambda}] = [-|g(\omega)_{1:\lambda})|, |g(\omega)_{1:\lambda})|]$, $g(\omega) \notin E_3$. That is $E_2 \cap E_3 = \emptyset$. So $E_3$ is disjoint with $E_1$ and $E_2$. Furthermore, for every $\omega \in \Omega$, except when $\omega_{1:\lambda} \neq 0$ which is a negligible subset of events, there exists a non negligible set of $(\omega_{i:\lambda})_{i \in [[1,\lambda]]}$ such that $\omega_{3:\lambda} \in ]-|\omega_{1:\lambda}|, |\omega_{1:\lambda}|[ ~~{\rm and }~~ \omega_{1:\lambda} < \omega_{2:\lambda} \rbrace$. So $E_3$ is a non negligible subset of $\Omega$, where $\Nlambda^2(\omega) > \Nlambda[2]^2(\omega)$. Hence $\E(\Nlambda^2(\omega)) > \E(\Nlambda[2]^2(\omega))$, which is the monotony of Lemma~\ref{lm:increasing_lambda_expectation}.

For $\lambda = 1$, $\Nmin[1] \sim \NN$ so $\E(\Nmin[1]^2) = 1$. For $\lambda = 2$ we have $\E(\Nmin[2]^2+\mathcal{N}_{2:2}^2) = 2\E(\NN^2) = 2$, and since the normal distribution is symmetric $\E(\Nmin[2]^2) = \E(\mathcal{N}_{2:2}^2)$, hence $\E(\Nmin[2]^2) = 1$.
\qed\end{proof}

We can now link Proposition \ref{pr:sigrate} and Lemma \ref{lm:increasing_lambda_expectation} into the following theorem:

\begin{theorem} \label{th:convsig}
On linear functions, for $\lambda\ge3$, the step-size of the $(1, \lambda)$-CSA-ES without cumulation ($c=1$) diverges geometrically almost surely and in expectation at the rate $1/(2d_\sigma n)(\E(\Nmin^2) - 1)$, i.e.
\begin{equation}\label{eq:divas}
	\inv{t} \ln\left(\frac{\sig}{\sig[0]}\right) 
	\;
	\overset{a.s.}{\underset{t \rightarrow \infty}{\longrightarrow}} 
	\;
	\E\left(\ln\left(\frac{\sigma_{t+1}}{\sig}\right)\right)
	\;=\; 
	\frac{1}{2d_{\sigma}n} \left( \E \left( \Nlambda^2 \right) - 1 \right) 
	\enspace.
	\end{equation}

For $\lambda=1$ and $\lambda=2$, without cumulation, the logarithm of the step-size does an additive unbiased random walk i.e.
$\ln \sigma_{t+1} = \ln \sigma_{t} + W_{t} $ where $E[W_{t}]=0$. More precisely $W_{t} \sim 1/(2d_\sigma)(\chi_n^2/n - 1)$ for $\lambda=1$, and $W_{t} \sim 1/(2d_\sigma)((\Nmin[2]^2 + \chi_{n-1}^2)/n - 1)$ for $\lambda =2$, where $\chi_k^2$ stands for the chi-squared distribution with $k$ degree of freedom.
\end{theorem}
\begin{proof}
For $\lambda > 2$, from Lemma~\ref{lm:increasing_lambda_expectation} we know that $\E(\Nmin^2) > \E(\Nmin[2]^2) = 1$. Therefore $\E(\Nmin^2) - 1 > 0$, hence Eq.~\eqref{eq:divas} is strictly positive, and with Proposition~\ref{pr:sigrate} we get that the step-size diverges geometrically almost surely at the rate $1/(2d_\sigma)(\E(\Nmin^2) - 1)$. 

With Eq.~\ref{eq:stepsizechange} we have $\ln(\sigma_{t+1}) = \ln(\sigma_t) + W_t$, with $W_t = 1/(2d_\sigma)(\|\chosenstep\|^2/n - 1)$. For $\lambda = 1$ and $\lambda = 2$, according to Lemma~\ref{lm:increasing_lambda_expectation}, $\E(W_t) = 0$. Hence $\ln(\sigma_t)$ does an additive unbiased random walk. Furthermore $\|\chosenstepdistribution\|^2 = \Nmin^2 + \chi_{n-1}^2$, so for $\lambda = 1$, since $\Nmin[1] = \NN$, $\|\chosenstepdistribution\|^2 = \chi_n^2$.
\qed\end{proof}

\subsection{Geometric divergence  of $([X_t]_1)_{t \in \N}$} \label{subsec:Xnoc}

As the selection occurs only on the first dimension, if there is geometric divergence for $\bs{X}_t$, it is on $\firstdim{\x}$. From Eq~\eqref{eq:selection2}
\begin{equation*}
\ln \left| {\xf[t+1] \over \xf} \right| = \ln\left|1 + {\sigma_t \over \xf} \firstdim{\chosenstep} \right| \enspace .
\end{equation*}
 Summing previous equation from $0$ till $t-1$ and dividing by $t$ gives us that
\begin{equation}\label{eq:xrec}
\inv{t}\ln \left| {\xf \over \xf[0]}\right| = \inv{t}\sum_{k=0}^{t-1} \ln\left|1 + {\sigma_k \over \xf[k]} \firstdim{\chosenstep} \right| \enspace.
\end{equation}

Although it is not obvious at first sight, it is important to take the logarithm, as we intuitively know that the speed of $\sigma_t$ and the speed of $\x$ are connected. The divergence rate of $\sigma_t$ being log-linear, so should be the one of $\x$. Let $Z_{-1} = 0$, and $Z_t = {\xf[t+1] - \xf[0] \over \sigma_t}$ for $t \geq 0$.
\begin{align*}
Z_{t+1} &= {\xf[t+2] - \xf[0] \over \sigma_{t+1}} = {\xf[t+1] - \xf[0] + \sigma_{t+1} \fchosenstep[t+1] \over \sigma_{t+1}}\\
Z_{t+1} &= \frac{Z_t}{\stepsizechange} + \firstdim{\chosenstep[t+1]}
\end{align*}
using that  $\sig[t+1] = \sig \stepsizechange$. According to Lemma \ref{lem:selectedstep} $\suitet{\chosenstep}$ is independent over time. As $\stepsizechange = \exp((\|\chosenstep\|^2/n - 1)/(2d_\sigma))$, $\suitet{\stepsizechange}$ is also independent over time. Therefore, $\zchain = \suite{Z}$, is a Markov chain.

By introducing $\zchain$ in Eq~\eqref{eq:xrec}, we obtain:
\begin{align}
\inv{t}\ln \left| {\xf \over \xf[0]}\right| &= \inv{t}\sum_{k=0}^{t-1} \ln\left|1 + {\sigma_{k-1} \stepsizechange[k-1] \over  \xf[k] } \firstdim{\chosenstep[k]}  \right| \nonumber \\
     &=     \inv{t}\sum_{k=0}^{t-1} \ln\left|1 + { \stepsizechange[k-1]  \over Z_{k-1}} \firstdim{\chosenstep[k]} \right| \nonumber \\
 &= \inv{t}\sum_{k=0}^{t-1} \ln\left| \frac{\frac{Z_{k-1}}{\stepsizechange[k-1]}+\chosenstep[k]}{\frac{Z_{k-1}}{\stepsizechange[k-1]}} \right| \nonumber \\
 &= \inv{t}\sum_{k=0}^{t-1} \left(\ln\left| Z_k \right| - \ln\left|Z_{k-1}\right| + \ln\left| \stepsizechange[k-1] \right| \right) \label{eq:XZ}
\end{align}
The right hand side of this equation reminds us again of the law of large numbers. There is no independence over time, but $\zchain$ being a Markov chain, if it follows some specific stability properties of Markov chains, then a law of large numbers may apply.

\subsubsection{Study of the Markov chain $\zchain$}

To apply a law of large numbers to a Markov chain, it has to satisfies some stability properties: in particular, the Markov chain $\pchain$ has to be $\varphi$-irreducible, that is, there exists a measure $\varphi$ such that every Borel set  $A$ of $\R^{n}$  with $\varphi(A)>0$ has a positive probability to be reached in a finite number of steps  by $\pchain$ starting from any $\bs{p}_{0} \in \R^{n}$. In addition, the chain $\pchain$ needs to be (i) positive, that is the chain admits an invariant probability measure $\pi$, i.e., for any borelian $A$, $\pi(A) = \int_{\R^n} P(x, A) \pi(dx)$ with $P(x,A)$ being the probability to transition in one time step from $x$ into $A$, and (ii) Harris recurrent which means for any borelian $A$ such that $\varphi(A) > 0$, the chain $\pchain$ visits $A$ an infinite number of times with probability one. Under those conditions, $\pchain$ satisfies a law of large numbers, more precisely:
\begin{lemma}\cite[17.0.1]{markovtheory} \label{lm:lln}
Suppose that $\bs{\Phi}$ is a positive Harris chain with stationary measure $\pi$, and let $g$ be a $\pi$-integrable function that is such that $\pi(|g|) = \int_{\R^n} |g(x)| \pi(dx) < \infty$. Then 
\begin{equation} \label{eq:lln}
1/t \sum_{k=1}^t g(\bs{\Phi}_k) \overset{a.s}{\underset{t\rightarrow \infty}{\longrightarrow}} \pi(g) \enspace .
\end{equation}
\end{lemma}

To show that a Markov defined in a space $X$ is positive Harris recurrent, we generally show that the chain follows a so-called drift condition over a small set, that is for a function $V$, an inequality over the drift operator $\Delta V : x\mapsto \int_X V(y)P(x, dy) - V(x)$. A small set is a borel set such that there exists a $m \in \N^*$ and a non-trivial measure $\nu_m$ on $\beta(X)$ such that for all $x \in C$, $B \in \beta(X)$, $P^m(x,B) \geq \nu_m(B)$. The set $C$ is then called a $\nu_m$-small set. The chain also needs to be aperiodic, that is there is no $d$-cycle, that is disjoint Borel sets $(D_i)_{i \in [[1, d]]}$, such that for $x \in D_i$, $P(x, D_{i+1}) = 1$ for $i = 0 \cdots d-1 (mod d)$, and $[\cup_{i=1}^d]^c$ is $\varphi$-negligible. If there exists a $\nu_1$-small-set $A$ such that $\nu_1(A)>0$, then the chain is strongly aperiodic (and therefore aperiodic). We then have the following lemma.

\begin{lemma}\cite[14.0.1]{markovtheory} \label{lm:ergodrift}
Suppose that the chain $\bs{\Phi}$ is $\varphi$-irreductible and aperiodic, and $f \geq 1$ a function on $X$. Let us assume that there exists $V$ some extended-valued non-negative function finite for some $x_0 \in X$, a small set $C$ and $b \in \R$ such that
\begin{equation} \label{eq:drift}
\Delta V(x) \leq -f(x) + b \mathbf{1}_C(x) \enspace, x \in X.
\end{equation}
Then the chain $\bs{\Phi}$ is positive Harris recurrent with invariant probability measure $\pi$ and
\begin{equation}
\pi(f) = \int_X \pi(dx) f(x) < \infty \enspace.
\end{equation}
\end{lemma}

To prove the irreducibility, aperiodicity and to exhibit the small sets of the Markov chain $\zchain$ through its transition kernel would be difficult. Instead, it can be done by showing some properties of its underlying control model. In our case, the model associated to $\zchain$ is called a non-linear state space model. We will, in the following, define this non-linear state space model and some of its properties. 

Suppose $\bs{X} = \lbrace \bs{X}_k \rbrace$, $\bs{X}_k \in \mathcal{X}$. If there is a smooth function ($C^\infty$) $F$ such that $\bs{X}_{k+1} = F(\bs{X}_k, \bs{W}_{k+1})$ with $(\bs{W}_{i})_{i \in \N}$ being a sequence of i.i.d. random variables, whose marginal distribution $\Gamma$ possesses a semi lower-continuous density $\gamma_w$ which is supported on an open set $O_w$; then $\bs{X}$ is called a non-linear state space model driven by $F$ or NSS($F$) model, with control set $O_w$.

We define its associated control model CM($F$) the deterministic system $x_k = F_k(x_0, u_1, \cdots, u_k)$, where $F_k$ is given by $F_k(x_0, u_1, \cdots, u_k) = F(F_{k-1}(x_0, u_1, \cdots, u_{k-1}), u_k)$, and $F_0(x_0) = x_0$, provided that $(u_i)_{i \in \N}$ lies in the control set $O_w$.

For a point $\bs{x} \in \mathcal{X}$, and $k \in \N$ we define $A_+^k(\bs{x}) = \lbrace F_k(\bs{x}, u_1, \cdots, u_k) | u_i \in O_w \hspace*{2mm} \forall i \in \N \rbrace$, the set of points reachable from $\bs{x}$ after $k$ steps of time. And $A_+(\bs{x}) = \bigcup_{i \in \N} A_+^i(\bs{x})$.

The associated control model CM($F$) is called forward accessible if for each $\bs{x}_0 \in \mathcal{X}$, the set $A_+(\bs{x_0})$ has non empty-interior.

Let $E$ be a subset of $\mathcal{X}$. We note $A_+(E) = \bigcup_{\bs{x} \in E} A_+(\bs{x})$, and we say that $E$ is invariant if $A_+(E) \subset E$. We call a set minimal if it is closed, invariant, and does not strictly contain any closed and invariant subset. Restricted to a minimal set, a Markov chain has strong properties, as stated in the following lemma.

\begin{lemma}\cite[7.2.4, 7.3.5]{markovtheory} \label{lm:controlmodel}
Let $M \subset \mathcal{X}$ be a minimal set for CM($F$). If CM($F$) is forward accessible then the NSS($F$) model restricted to $M$ is an open set irreducible T-chain.

Furthermore, if the control set $O_w$ and $M$ are connected, and that $M$ is the unique minimal set of the CM($F$), then the NSS($F$) model is a $\psi$-irreducible aperiodic T-chain for which every compact set is a small set.
\end{lemma}

We can now prove the following lemma:
\begin{lemma} \label{lm:zproperties}
The Markov chain $\zchain$ is open set and $\psi$-irreducible, aperiodic, and compacts of $\R$ are small-sets.
\end{lemma}

\begin{proof}
This is exactly the result of Theorem \ref{lm:controlmodel} when all conditions are fulfilled. We then have to show the right properties of the underlying control model.

If we note $F( X_k, \bs{W}_{k+1}) = X_k \exp\left(-1/2d_\sigma \left(\|\bs{W}_{k+1}\|^2/n - 1\right)\right) + \firstdim{\bs{W}_{k+1}}$, then we do have $\z[t+1] = F(\z, \chosenstep)$. The function $F$ is smooth (it is not smooth along the instances $\xi_{t,i}$, but along the chosen step). Furthermore, the distribution of $\chosenstep$ admits a continuous density, whose support is $\R^n$. Therefore the process $\zchain$ is a NSS($F$) model of control set $\R^n$.

We now have to show that the associated control model is forward accessible. Let $z \in \R$. When $\fchosenstep \rightarrow \pm \infty$, $F(z, \chosenstep) \rightarrow \pm \infty$. As $F$ is continuous, for the right value of $\fchosenstep$ any point of $\R$ can be reach. Therefore for any $z\in \R$, $A_+(z) = \R$. The set $\R$ has a non-empty interior, so the CM($F$) is forward accessible.

As from any point of $\R$, all of $\R$ can be reached, the only invariant set is $\R$ itself. It is therefore the only minimal set. Finally, the control set $O_w = \R^n$ is connected, and so is the only minimal set, so all the conditions of Lemma \ref{lm:controlmodel} are met. So the Markov chain $\zchain$ is $\psi$-irreducible, aperiodic, and compacts of $\R$ are small-sets.
\qed\end{proof}

We may now show Foster-Lyapunov drift conditions to ensure the Harris positive recurrence on the chain $\zchain$. In order to do so, we will need the following Lemma:
\begin{lemma} \label{lm:etaesp}
Let $\exp (-\frac{1}{2 d_\sigma}( \frac{\|\bs{\xi}^\star\|^2}{n} - 1))$ be denoted $\stepsizechange[\cdot]$. For all $\lambda > 2$ there exists $\alpha>0$ such that
\begin{equation} \label{eq:etaespassumption}
\E \left( {\stepsizechange[\cdot]}^{-\alpha}\right) - 1 < 0 \enspace.
\end{equation}
\end{lemma}

\begin{proof}
Using the Taylor series of the exponential function we have
\begin{align*}
\E \left( {\stepsizechange[\cdot]}^{-\alpha}\right) &= \E \left( \exp \left({-\frac{\alpha}{2 d_\sigma}\left( \frac{\|\bs{\xi}^\star\|^2}{n} - 1\right)}\right)  \right) \\
&= \E \left( \sum_{i=0}^\infty \frac{\left(-\frac{\alpha}{2 d_\sigma}\left( \frac{\|\bs{\xi}^\star\|^2}{n} - 1\right) \right)^i}{i!} \right) \\
&= 1 - \alpha \left( \frac{1}{2d_\sigma n} \left( \E \left( \Nlambda^2 \right) - 1   \right)   - o\left(\alpha^2 \right) \right) \enspace .
\end{align*}
According to Lemma~\ref{lm:increasing_lambda_expectation} $\E \left( \Nlambda^2 \right) > 1$ for $\lambda > 2$, so when $\alpha$ goes to $0$ we have $\E \left( {\stepsizechange[\cdot]}^{-\alpha}\right) < 1$.
\qed \end{proof}

We are now ready to prove the following lemma:

\begin{lemma} \label{lm:zharris}
The Markov chain $\zchain$ is Harris recurrent positive, and admits a unique invariant measure $\mu$ such that for $f$: $x \mapsto |x|^\alpha \in \R$, $\mu(f) = \int_{\R} \mu(dx) f(x) < \infty$, with $\alpha$ such that Eq.~\eqref{eq:etaespassumption} holds true.
\end{lemma} 

\begin{proof}
By using Lemma \ref{lm:zproperties} and Lemma \ref{lm:ergodrift}, we just need the drift condition \eqref{eq:drift} to prove Lemma \ref{lm:zharris}. Let $V$ be such that for $x\in\R$, $V(x) = |x|^\alpha + 1$.
\begin{align*}
\Delta V(x) &= \int_\R P(x,dy) V(y) - V(x)\\
&= \int_\R P\left(\frac{x}{\stepsizechange[\cdot]}  + \fchosenstep[\cdot] \in dy \right)(1+ |y|^\alpha) - (1+|x|^\alpha)\\
&= \E\left( \left|\frac{x}{\stepsizechange[\cdot]}  + \fchosenstep[\cdot] \right|^\alpha\right) - |x|^\alpha\\
&\leq |x|^\alpha \E \left( {\stepsizechange[\cdot]}^{-\alpha} - 1  \right) +  \E\left(  \fchosenstep[\cdot]^\alpha \right)\\
\frac{\Delta V(x)}{V(x)} &= \frac{|x|^\alpha}{1 + |x|^\alpha} \E \left( {\stepsizechange[\cdot]}^{-\alpha} - 1  \right) + \frac{1}{1+|x|^\alpha}\E\left( \fchosenstep^\alpha \right)\\
\underset{|x| \longrightarrow \infty}{\lim} \frac{\Delta V(x)}{V(x)} 
	&=  \E \left( {\stepsizechange[\cdot]}^{-\alpha} - 1  \right)
\end{align*}

We take $\alpha$ such that Eq.~\eqref{eq:etaespassumption} holds true (as according to Lemma~\ref{lm:etaesp}, there exists such a $\alpha$). As $\E({\stepsizechange[\cdot]}^{-\alpha} - 1 ) < 0$, there exists $\epsilon > 0$ and $M > 0$ such that for all $|x| \geq M$, $\Delta V/V(x) \leq - \epsilon$. Let $b$ be equal to $\E(\fchosenstep^2) + \epsilon V(M)$. Then for all $|x| \leq M$, $\Delta V(x) \leq -\epsilon V(x) + b$. Therefore, if we note $C = [-M, M]$, which is according to Lemma~\ref{lm:zproperties} a small-set, we do have $\Delta V(x) \leq -\epsilon V(x) + b \mathbf{1}_C(x)$ which is Eq.~\eqref{eq:drift} with $f = \epsilon V$. Therefore from Lemma~\ref{lm:ergodrift} the chain $\zchain$ is positive Harris recurrent with invariant probability measure $\mu$, and $\epsilon V$ is $\mu$-integrable. As $\int_{\R} \mu(dx) |x|^\alpha = 1/\epsilon \int_{\R} \mu(dx) \epsilon V(x) - 1 < \infty$, the function $x \mapsto |x|^\alpha$ is also $\mu$-integrable.

\qed\end{proof}

In order to use Lemma~\ref{lm:lln} on $\zchain$ with the function $g : x \mapsto \E\left(\ln\left| x/x - \fchosenstep[\cdot] \right|\right)$, we must prove that this function is $\mu$-integrable, that is $\int_\R g(u)\mu(du) < \infty$. To do so we will need the following lemma on the existence of moments for stationary Markov chains:
\begin{lemma} \label{lm:momentmarkov}
Let $\zchain$ be a Harris-recurrent Markov chain with stationary measure $\mu$, on a state space $(S, \mathcal{F})$, with $\mathcal{F}$ is $\sigma$-field of subsets of $S$. Let $f$ be a positive measurable function on $S$.

In order that $\int_S f(z) \mu(dz) < \infty$, it suffices that for some set $A \in \mathcal{F}$ such that $0 < \mu(A)$ and $\int_A f(z)\mu(dz) < \infty$, and some measurable function g with $g(z) \geq f(z)$ for $z \in A^c$,
	\begin{enumerate}
		\item $$\int_{A^c} P(z, dy)g(y) \leq g(z) - f(z) \hspace*{3mm}, \hspace*{1mm} \forall x \in A^c$$
		\item $$\sup_{z \in A} \int_{A^c} P(z, dy) g(y) < \infty$$.
	\end{enumerate}
\end{lemma}




We may now prove the following theorem:
\begin{theorem}
	On linear functions, for $\lambda \geq 3$, the absolute value of the first dimension of the parent point in the $(1,\lambda)$-CSA-ES without cumulation ($c = 1$) diverges geometrically almost surely at the rate of $1/(2d_\sigma n)\E(\Nlambda^2 - 1)$, i.e.
\begin{equation}\label{eq:Xdiv}
\frac{1}{t}\ln\left|\frac{\xf}{\xf[0]}\right| \overset{a.s}{\underset{t \rightarrow \infty}{\longrightarrow}} \frac{1}{2d_\sigma n}\left( \E\left(\Nlambda^2\right) - 1 \right) \enspace .
\end{equation}
\end{theorem}

\begin{proof}
We will first prove here that the function $g : x \mapsto \ln |x|$ is $\mu$-integrable.
From Lemma~\ref{lm:zharris} we know that the function $f : x \mapsto |x|^\alpha$ is $\mu$-integrable, and as for any $M>0$, and any $x \in [-M, M]^c$ there exists $K>0$ such that $K|x|^\alpha > |\ln|x||$, then $g \mathbf{1}_{A^c}$ is $\mu$-integrable, with $A = [-M, M]$. So what is left is to prove that $g \mathbf{1}_A$ is also $\mu$-integrable. We will now check the conditions to use Lemma~\ref{lm:momentmarkov}.

According to Lemma~\ref{lm:zproperties} the chain $\zchain$ is open-set irreducible, so $\mu(A^c)>0$. For $C>0$, if we take $h : z \mapsto C/\sqrt{|z|}$, with $M$ small enough we do have for all $z\in A^c$, $h(z) \geq |g(z)|$. Furthermore, if we study the inequality
\begin{align*}
	\int_{A^c} P(z, dy)h(y) &\leq h(z) - |g(z)|\\	
	\int_S P\left( \frac{z}{\stepsizechange[\cdot]} + \fchosenstep[\cdot] \in dy\right) \mathbf{1}_{A^c}(y) \frac{C}{\sqrt{|y|}} &\leq \frac{C}{\sqrt{|z|}} - |ln|z||\\	
	\E\left( \frac{1}{\sqrt{\left|\frac{z}{\stepsizechange[\cdot]} + \fchosenstep[\cdot]\right|}} \mathbf{1}_{A^c}\left(\frac{z}{\stepsizechange[\cdot]} + \fchosenstep[\cdot]  \right) \right) &\leq \frac{1}{\sqrt{|z|}} - \frac{|\ln|z||}{C}
\end{align*}
We can increase $C$ up until $|\ln|z||/C$ is negligible compared to $1/\sqrt{|z|}$, and we can decrease $M$ to make 
$\E\left( 1/\sqrt{\left|z/\stepsizechange[\cdot] + \fchosenstep[\cdot]\right|} \mathbf{1}_{A^c}\left(z/\stepsizechange[\cdot] + \fchosenstep[\cdot]  \right) \right)$ as small as we would like it to be, as it decreases the size of $A^c$, so the inequality holds if we choose $M$ and $1/C$ small enough.
The second inequality for Lemma~\ref{lm:momentmarkov} holds as well:
\begin{equation*}
\int_{A^c} P(u, dv)h(v) \le \int_{A^c} \frac{C}{\sqrt{|v|}} dv = 4C\sqrt{M} < \infty
\end{equation*}

Finally, according to Lemma~\ref{lm:zharris}, the chain $\zchain$ is Harris recurrent. So Lemma~\ref{lm:momentmarkov} shows that $g$ is $\mu$-integrable. This allows us to apply Lemma~\ref{lm:lln} to the function $g$: $1/t \sum_{k=1}^t g(z_k) \overset{a.s}{\underset{t \rightarrow \infty}{\longrightarrow}} \mu(g)$. 

With Lemma~\ref{lem:selectedstep} we can apply a strong law of large numbers upon $1/t\sum_{k=0}^{t-1} \ln|\stepsizechange[k-1]| = 1/t \sum_{k=0}^{t-1} 1/(2d_\sigma)(\chosenstep[k-1]/n - 1)$, to get as in the proof of Proposition~\ref{pr:sigrate} $1/(2d_\sigma n)(\E(\Nlambda^2) - 1)$.

By inserting these results into Eq.~\eqref{eq:XZ}, we get that $1/t \ln |\xf/\xf[0]| \overset{a.s}{\underset{t \rightarrow \infty}{\longrightarrow}} \mu(g) - \mu(g) + 1/(2d_\sigma n)(\E(\Nlambda^2) - 1)$, which with Lemma~\ref{lm:increasing_lambda_expectation} is strictly positive for $\lambda \geq 3$.
\qed\end{proof}

\section{Divergence rate of CSA-ES with cumulation} \label{sec:withcumulation}

We are now investigating the $(1,\lambda)$-CSA-ES with cumulation, i.e. $0 < c < 1$.

According to Lemma~\ref{lem:selectedstep}, the random variables $(\chosenstep)_{t\in \N}$ are i.i.d., hence the path $\pchain = (\p)_{t \in \N}$ is a Markov chain.
By a recurrence on Eq.~\eqref{eq:chemin} we see that the path follows the following equation
\begin{equation} \label{eq:path}
\p = (1-c)^{t}\bs{p}_0 + \sqrt{c(2-c)}\sum_{k=0}^{t-1} (1-c)^k \underbrace{\chosenstep[t-1-k]}_{\text{i.i.d.}} \enspace .
\end{equation}
For $i \neq 1$, $[\chosenstep]_i \sim \NNN(0, 1)$ and,
as also $[\bs{p}_0]_i \sim \NNN(0,1)$,
by recurrence $[\bs{p_t}]_i \sim \NNN(0,1)$ for all $t \in \N$. 
For $i=1$ with cumulation ($c <1$), the influence of $[\bs{p}_0]_1$ vanishes with $(1-c)^{t}$. Furthermore, as from Lemma~\ref{lem:selectedstep} the sequence $(\fchosenstep])_{t \in \N}$ is independent, we get by applying the Kolgomorov's three series theorem that the series $\sum_{k=0}^{t-1} (1-c)^k \fchosenstep[t-1-k]$ converges almost surely. Therefore, the first component of the path becomes distributed as the random variable $[\bs{p}_\infty]_1 = \sqrt{c(2-c)} \sum_{k=0}^\infty (1-c)^k [\chosenstep[k]]_1$ (by re-indexing the variable $\bs{\xi}_{t-1-k}^\star$ in $\bs{\xi}_k^\star$, as the sequence $(\chosenstep)_{t \in \N}$ is i.i.d.). 

As in Subsection~\ref{subsec:Xnoc} we will show that $\pchain$ has the right stability properties to apply a law of large numbers to it. First we will extract from $\pchain$ the part of interest as stated in the following lemma.

\begin{lemma} \label{lm:pfdim}
	On linear functions, for any $\lambda$ the step-size of the $(1, \lambda)$-CSA-ES follows   almost surely
\begin{equation} \label{eq:pfdim}
\frac{1}{t}\ln\left(\frac{\sigma_t}{\sigma_0}\right) - \frac{c}{2d_\sigma n}\left( \frac{1}{t} \sum_{i=1}^t ([\bs{p}_i]_1^2 - 1\right) \overset{a.s}{\underset{t \rightarrow \infty}{\longrightarrow}} 0 \enspace ,
\end{equation}
and in expectancy
\begin{equation} \label{eq:esppfdim}
\E\left(\ln\left(\frac{\sigma_{t+1}}{\sigma_t}\right)\right) = \frac{c}{2d_\sigma n}\left( \E\left([\bs{p}_{t+1}]_1^2\right) - 1\right)
\end{equation}
\end{lemma}

\begin{proof}
We separate Eq.~\eqref{eq:sigcumul} over the dimensions, which gives us that $1/t \ln (\sigma_t / \sigma_0) = c/(2d_\sigma n)( \sum_{i=1}^n 1/t \sum_{j=1}^t [\bs{p}_j]_i^2 - n)$, so $1/t \ln (\sigma_t / \sigma_0) - c/(2d_\sigma n)( 1/t \sum_{j=1}^t [\bs{p}_j]_1^2 - 1) =  \sum_{i=2}^n c/(2d_\sigma n)( 1/t \sum_{j=1}^t [\bs{p}_j]_i^2 - 1)$. As for $i\neq 1$, $[\bs{p}_0]_i \NL$ and $[\bs{\xi}_0^\star]_i$ has no selection pressure, then  $[\bs{p}_1]_i \NL$, and per recurrence $[\bs{p}_k]_i \NL$ for any $k \in \N$. Therefore, we can apply the strong law of large numbers and $1/t \sum_{j=1}^t [\bs{p}_j]_i^2 \overset{a.s}{\underset{t \rightarrow \infty}{\longrightarrow}} 1$, which gives us Eq.~\eqref{eq:pfdim}.

The same reasoning over Eq.~\eqref{eq:stepsizechange} gives Eq.~\eqref{eq:esppfdim}.
\qed \end{proof}

The part of $\pchain$ left to analyse is its first dimension $\fpchain = ([\bs{p}_i]_1)_{i \in \N}$. We start the study of $\fpchain$ with the following lemma.

\begin{lemma} \label{lm:pproperties}
The Markov chain $\fpchain$ is $\varphi$-irreducible, aperiodic, and compacts of $\R$ are small-sets.
\end{lemma}

\begin{proof}
We have the following transition kernel:
\begin{equation*}
P \left( p, A \right) = \int_{\R} \mathbf{1}_A \left( (1-c)p + \sqrt{c(2-c)} u \right) P(\Nlambda = u) du \enspace.
\end{equation*}
With a change of variables $\tilde{u} = (1-c)p \sqrt{c(2-c)}u$, we get that
\begin{equation*}
P \left( p, A \right) = \inv{\sqrt{(2-c)c}}  \int_{\R} \mathbf{1}_A \left( \tilde{u} \right) P\left(\Nlambda = \frac{\tilde{u}-(1-c)p}{\sqrt{(2-c)c}}\right) d\tilde{u} \enspace .
\end{equation*}
As $P\left(\Nlambda = x\right) > 0$ for all $x \in \R$, for all $A$ non-$\mu_{Leb}$-negligible we have $P(p, A) > 0$, thus the chain $\fpchain$ is $\mu_{Leb}$-irreducible.

Furthermore, if we take $C$ a non-$\mu_{Leb}$-negligible compact of $\R$, and $\nu_C$ a measure such that for $A$ a borel set of $\R$, \\$\nu_C(A) = \inv{\sqrt{(2-c)c}}  \int_{\R} \mathbf{1}_A \left( \tilde{u} \right) \underset{p\in C}{\min}P\left(\chosenstep = \left(\tilde{u}-(1-c)p\right)/\left(\sqrt{(2-c)c}\right)\right) d\tilde{u}$, we see that $P(p, A) \geq \nu_C(A)$ for all $p \in \R$, while $\nu_C$ is not a trivial measure (indeed, $P(\Nlambda = x) > k > 0$ for all $x \in C$); Therefore compact sets of $\R$ are small sets for $\fpchain$. Finally, $\nu_C(C) > 0$, so the chain $\fpchain$ is strongly aperiodic.
\qed\end{proof}

We use this new lemma with Lemma~\ref{lm:ergodrift} to prove what is needed to apply the law of large numbers on $\fpchain$.
\begin{lemma} \label{lm:pharris}
  The chain $\fpchain$ is Harris recurrent positive with invariant measure $\mu_{path}$, and the function $x \mapsto x^2$ is $\mu_{path}$-integrable.
\end{lemma}

\begin{proof}
We now have to get the right drift condition for the chain. Let $V : x \mapsto x^2 + 1$. 
\begin{align*}
\Delta V (x) &= \int_{\R} V(y) P(x, dy)  - V(x)\\
\Delta V (x) &= \int_{\R} \left( y^2 + 1 \right) P \left( (1-c)x + \sqrt{c(2-c)}\fchosenstep[\cdot] \in dy \right) - \left(x^2 + 1 \right)\\
\Delta V(x) &= \E \left(\left(1-c)x + \sqrt{c(2-c)}\fchosenstep[\cdot]\right)^2 + 1\right) - x^2 - 1\\
\Delta V(x) &\leq ((1-c)^2 - 1) x^2 + 2|x|\sqrt{c(2-c)}\E\left({\fchosenstep[\cdot]}\right) + c(2-c)\E\left({\fchosenstep[\cdot]}^2\right) \\
\frac{\Delta V(x)}{V(x)} &\leq -c(2-c) \frac{x^2}{1 + x^2} + \frac{2|x|\sqrt{c(2-c)}}{1+ x^2} \E\left({\left|\fchosenstep[\cdot]\right|}\right) + \frac{c(2-c)}{1+x^2}\E\left({\fchosenstep[\cdot]}^2\right)\\
\lim_{|x| \rightarrow \infty} \frac{\Delta V(x)}{V(x)} &\leq -c(2-c)
\end{align*}

As $0 < c \leq 1$, $c(2-c)$ is strictly positive and therefore, for $\epsilon>0$ there exists $C = [-M, M]$ with $M>0$ such that for all $x \in C^c$, $\Delta V(x)/V(x) \leq -\epsilon$. If we take $b = \epsilon V(M) + 2M\sqrt{c(2-c)}\E(|\fchosenstep[\cdot]|) + c(2-c)\E({\fchosenstep}^2)$, then for all $x \in C$ we have $\Delta V(x) \leq b$. Hence the drift condition $\Delta V(x) \leq -\epsilon V(x) + b \mathbf{1}_C$ is satisfied for all $x \in \R$.

According to Lemma~\ref{lm:pproperties} the chain $\fpchain$ is $\varphi$-irreducible and aperiodic, so with Lemma~\ref{lm:ergodrift} it is positive Harris recurrent, with invariant measure $\mu_{path}$, and $V$ is $\mu_{path}$-integrable. Therefore the function $x \mapsto x^2$ is also $\mu_{path}$-integrable.
\end{proof}

To obtain an equality between the rate we get through almost sure divergence, and the rate in expectation, we need to define the $f$-norm, which for a signed measure $\nu$ and a function $f \geq 1$ is equal to $\|\nu\|_f = \sup_{g: |g| \leq f} |\nu(g)|$, and we need the following lemma.

\begin{lemma}\cite[14.3.5]{markovtheory} \label{lm:as_esp}
Suppose $\bs{\Phi}$ is an aperiodic positive Harris chain on a space $\mathcal{X}$ with stationary measure $\pi$, and that there exists some non-negative function $V$, a function $f \geq 1$, a small-set $C$ and $b \in \R$ such that for all $x \in \mathcal{X}$, $\Delta V(x) \leq -f(x) + b \mathbf{1}_C(x)$. Then for all initial probability distribution $\nu$, $\|\nu P^n - \pi\|_f \underset{t \rightarrow \infty}{\longrightarrow} 0$.
\end{lemma}

We now obtain geometric divergence of the step-size and get an explicit estimate of the expression of the divergence rate.
\begin{theorem} \label{th:geometricdivergencecumul} 
The step-size of the $(1,\lambda)$-CSA-ES with $\lambda\ge2$ diverges geometrically fast if $c<1$ or $\lambda\ge3$. Almost surely and in expectation we have for $0<c\le1$,
\begin{equation}\label{eq:sigcumulconv}\inv{t} \ln\left( {\sig \over \sig[0]} \right) \overset{}{\underset{t \rightarrow \infty}{\longrightarrow}} 
\frac{1}{2d_\sigma n} 
\underbrace{
\left( 2(1-c)\,\E\left(\Nlambda\right)^2  + c\left(\E\left(\Nlambda^2\right) - 1\right) \right)
}_{>0 \mbox{ for } \lambda\ge3 \mbox{ and for } \lambda = 2 \mbox{ and } c < 1 } 
\enspace.
\end{equation}
\end{theorem}
\begin{proof}
We will start by the convergence in expectation.
%
From Eq.~\eqref{eq:esppfdim} we see that the part to develop is $\E(\pdim[1]^2)$. By recurrence  $\pdim[1] = (1-c)^{t+1}[\bs{p}_0]_1 + \sqrt{c(2-c)}\sum_{i=0}^t (1-c)^i \firstdim{\chosenstep[t-i]}$. When $t$ goes to infinity, the influence of $[\bs{p}_0]_1$ in this equation goes to $0$ with $(1-c)^{t+1}$, so we can remove it when taking the limit:
\begin{equation}
\lim_{t \rightarrow \infty}\E\left( \pdim[1]^2 \right) = \lim_{t \rightarrow \infty} \E\bigg( \bigg(\sqrt{c(2-c)}  \sum_{i=0}^t (1-c)^i \firstdim{\chosenstep[t-i]} \bigg)^2\bigg)
\end{equation}
We will now develop the sum with the square, such that we have either a product $\firstdim{\chosenstep[t-i]} \firstdim{\chosenstep[t-j]}$ with $i \neq j$, or $\firstdim{\chosenstep[t-j]}^2$. This way, we can separate the variables by using Lemma~\ref{lem:selectedstep} with the independence of $\chosenstep[i]$ over time. To do so, we use the development formula $(\sum_{i=1}^n a_n)^2 = 2 \sum_{i=1}^n \sum_{j=i+1}^n a_i a_j + \sum_{i=1}^n a_i^2$. We take the limit of $\E( \pdim[1]^2 )$ and find that it is equal to
\begin{equation} \label{eq:devexppathsquare}
 \lim_{t \rightarrow \infty}c(2-c) \! \left( 2\sum_{i=0}^t \! \sum_{j=i+1}^t\!\! (1 \! - \! c)^{i+j} \!\!\!\! \underbrace{\E\left( \firstdim{\chosenstep[t-i]} \firstdim{\chosenstep[t-j]} \right)}_{=\E \firstdim{\chosenstep[t-i]} \E \firstdim{\chosenstep[t-j]}=\E [\Nmin]^{2}} \!\!\!\! + \sum_{i=0}^t (1 \! - \! c)^{2i} \underbrace{\E\left(\firstdim{\chosenstep[t-i]}^2\right)}_{=\E[\Nmin^2]} \right)
\end{equation}
Now the expected value does not depend on $i$ or $j$, so what is left is to calculate $\sum_{i=0}^t \sum_{j=i+1}^t (1-c)^{i+j}$ and $\sum_{i=0}^t (1-c)^{2i}$. We have $\sum_{i=0}^t \sum_{j=i+1}^t (1-c)^{i+j} = \sum_{i=0}^t (1-c)^{2i+1} \frac{1-(1-c)^{t-i}}{1-(1-c)}$ and when we separates this sum in two, the right hand side goes to $0$ for $t\to\infty$. Therefore, the left hand side converges to $\lim_{t \rightarrow \infty} \sum_{i=0}^t (1-c)^{2i+1}/c$, which is equal to $\lim_{t \rightarrow \infty} (1-c)/c \sum_{i=0}^t (1-c)^{2i}$. And $\sum_{i=0}^t (1-c)^{2i}$ is equal to $(1-(1-c)^{2t+2})/(1-(1-c)^2)$, which converges to $1/(c(2-c))$. So, by inserting this in Eq.~\eqref{eq:devexppathsquare} we get that $\E\left( \pdim[1]^2 \right) \underset{t \rightarrow \infty}{\longrightarrow} 2\frac{1-c}{c} \E\left(\Nmin\right)^2 + \E\left(\Nmin^2\right)$, which gives us the right hand side of Eq.~\eqref{eq:sigcumulconv}.

By summing $\E(\ln (\sigma_{i+1}/\sigma_i))$ for $i=0,\dots, t-1$ and dividing by $t$ we have the Cesaro mean $1/t \E(\ln (\sigma_{t}/\sigma_0))$ that converges to the same value that $\E(\ln (\sigma_{t+1}/\sigma_t))$ converges to when $t$ goes to infinity. Therefore we have in expectation Eq.~\eqref{eq:sigcumulconv}.

We will now focus on the almost sure convergence.
From Lemma~\ref{lm:pharris}, we see that we have the right conditions to apply Lemma~\ref{lm:lln} to the chain $\fpchain$ with the $\mu_{path}$-integrable function $g : x \mapsto x^2$. So $1/t\sum_{k=1}^t [\bs{p}_k]_1^2 \overset{a.s}{\underset{t \rightarrow \infty}{\longrightarrow}} \mu_{path}(g)$. With Eq.~\eqref{eq:pfdim} we obtain that $1/t \ln (\sigma_t/\sigma_0) \overset{a.s}{\underset{t \rightarrow \infty}{\longrightarrow}} c/(2d_\sigma n)(\mu_{path}(g) - 1)$.

We will now prove that $\mu_{path}(g) = \lim_{t \rightarrow \infty} \E(\pdim[1]^2)$. Let $\nu$ be the initial distribution of $[\bs{p}_0]1$, so we have $|\E(\pdim[1]^2) - \mu_{path}(g)| \leq \|\nu P^{t+1} - \mu_{path} \|_h$, with $h : x \mapsto 1 + x^2$. From the proof of Lemma~\ref{lm:pharris} and from Lemma~\ref{lm:pproperties} we have all conditions for Lemma~\ref{lm:as_esp}. Therefore $\|\nu P^{t+1} - \mu_{path} \|_h \underset{t \rightarrow \infty}{\longrightarrow} 0$, which shows that $\mu_{path}(g) = \lim_{t \rightarrow \infty} \E(\pdim[1]^2) = (2-2c)/c \E(\Nlambda)^2 + \E(\Nlambda^2)$.

According to Lemma~\ref{lm:increasing_lambda_expectation}, for $\lambda = 2$, $\E(\Nmin[2]^2) = 1$, so the RHS of Eq.~\eqref{eq:sigcumulconv} is equal to $(1-c)/(d_\sigma n) \E(\Nmin[2])^2$. The expected value of $\Nmin[2]$ is strictly negative, so the previous expression is strictly positive. Furthermore, according to Lemma~\ref{lm:increasing_lambda_expectation}, $\E(\Nmin^2)$ increases strictly with $\lambda$, as does $\E(\Nmin[2])^2$. Therefore we have geometric divergence for $\lambda \geq 2$ if $c < 1$, and for $\lambda \geq 3$.

\qed\end{proof}

From Eq.~\eqref{eq:selection2} we see that the behaviour of the step-size and of $\suite{\bs{X}}$ are directly related. Geometric divergence of the step-size, as shown in Theorem~\ref{th:geometricdivergencecumul}, means that also the movements in search space and the improvements on affine linear functions $f$ increase geometrically fast. Analyzing $\suite{\bs{X}}$ with cumulation would require to study a double Markov chain, which is left to possible future research.

\section{Study of the variations of $\ln\left({\sig[t+1]}/{\sig}\right)$} \label{sec:var}
The proof of Theorem~\ref{th:geometricdivergencecumul} shows that the step size increase converges to the right hand side of Eq.~\eqref{eq:sigcumulconv}, for $t\to\infty$. When the dimension increases this increment goes to zero, which also suggests that it becomes more likely that $\sigplus$ is smaller than $\sig$. To analyze this behavior, we study the variance of $\ln\left({\sig[t+1]}/{\sig}\right)$ as a function of $c$ and the dimension.

\begin{theorem} \label{th:var} The variance of $\ln\left({\sig[t+1]}/{\sig}\right)$ equals to
\begin{equation} \label{eq:var} \Var \left(\ln\left(\frac{\sigma_{t+1}}{\sigma_{t}}\right)\right) =
\frac{c^2}{4 d_\sigma^2 n^2}\left( \E\left(\firstdim{\pplus}^4 \right) - \E\left(\firstdim{\pplus}^2\right)^2 + 2(n-1) \right) 
\enspace.
\end{equation}
Furthermore, $\E\left(\pdim[1]^2\right) \underset{t \rightarrow \infty}{\longrightarrow} \E\left(\Nmin^2\right) + \frac{2-2c}{c}\E\left(\Nmin\right)^2$ and with $a = 1-c$ 
\begin{equation} \label{eq:devp4}
\lim_{t \to \infty}\E \left(\pdim[1]^4\right)= \frac{(1-a^2)^2}{1-a^4}\left( k_4 +k_{31} + k_{22} + k_{211} + k_{1111} \right) 
\enspace,
\end{equation}
where $k_{4}\!=\!\E\!\left(\Nmin^4\right)$, $k_{31}=  4\frac{a\left(1+a+2a^2\right)}{1-a^3}\E\left(\Nmin^3\right)\E\left(\Nmin\right)$, $k_{22}= 6\frac{a^2}{1-a^2}\E\left( \Nmin^2\right)^2$,\\ $k_{211}\!=\!12\frac{a^3(1+2a+3a^2)}{(1-a^2)(1-a^3)} \E\!\left(\Nmin^2\right)\! \E\!\left(\Nmin\right)^2$ and $k_{1111} = 24\frac{a^6}{(1-a)(1-a^2)(1-a^3)}\E\left(\Nmin\right)^4$.
\end{theorem}

\begin{proof} 
\begin{equation} \label{eq:proofvar}
\Var \left( \ln\left( \frac{\sigma_{t+1}}{\sigma_t} \right) \right) = \Var \left( \frac{c}{2d_\sigma}\left( \frac{\|\pplus \|^2}{n} - 1\right)\right) = \frac{c^2}{4d_\sigma^2 n^2}\!\!\!\! \underbrace{\Var\left(\|\pplus\|^2 \right)}_{\E\left(\|\pplus\|^4\right) - \E\left(\|\pplus\|^2 \right)^2 }
\end{equation}
The first part of $\Var(\|\pplus\|^2)$, $\E(\|\pplus\|^4)$, is equal to $\E((\sum_{i=1}^n \pdim^2 )^2 )$. We develop it along the dimensions such that we can use the independence of $[\pplus]_i$ with $[\pplus]_j$ for $i\neq j$, to get $\E( 2 \sum_{i=1}^n \sum_{j=i+1}^n \pdim[i]^2\pdim[j]^2  + \sum_{i= 1}^n \pdim[i]^4 )$. For $i \neq 1$ $\pdim[i]$ is distributed according to a standard normal distribution, so $\E\left(\pdim[i]^2\right) = 1$ and $\E\left(\pdim[i]^4\right) = 3$.
\begin{align*}
\E\left(\|\pplus\|^4\right) &= 2 \sum_{i=1}^n \sum_{j=i+1}^n \E\left(\pdim[i]^2\right) \E\left(\pdim[j]^2\right)  + \sum_{i=1}^n \E\left(\pdim[i]^4\right) \\
&= \left(2 \sum_{i=2}^n \sum_{j=i+1}^n 1\right) + 2\sum_{j=2}^n \E\left(\pdim[1]^2\right) + \left(\sum_{i=2}^n 3\right) + \E\left(\pdim[1]^4 \right)\\
&= \left( 2 \sum_{i=2}^n (n-i) \right) + 2(n-1) \E \left(\pdim[1]^2\right) + 3(n-1) + \E\left(\pdim[1]^4 \right)\\
&= \E\left(\pdim[1]^4 \right) + 2(n-1)\E \left(\pdim[1]^2\right) + (n-1)(n+1)
%
\end{align*}
The other part left is $\E(\|\pplus\|^2 )^2$, which we develop along the dimensions to get $\E ( \sum_{i=1}^n \pdim^2  )^2 = ( \E(\pdim[1]^2) + (n-1) )^2$, which equals to $\E(\pdim[1]^2)^2 + 2(n-1)\E(\pdim[1]^2) + (n-1)^2$.
So by subtracting both parts we get \\$\E(\|\pplus\|^4) - \E(\|\pplus\|^2 )^2 = \E(\pdim[1]^4 )  - \E(\pdim[1]^2)^2 + 2(n-1)$, which we insert into Eq.~\eqref{eq:proofvar} to get Eq.~\eqref{eq:var}.

The development of $\E(\pdim[1]^2)$ is the same than the one done in the proof of Theorem~\ref{th:geometricdivergencecumul}, that is $\E(\pdim[1]^2) = (2-2c)/c\E(\Nlambda)^2 + \E(\Nlambda^2)$. We now develop $\E(\pdim[1]^4)$. We have $\E(\pdim[1]^4) = \E(((1-c)^t [\p[0]]_1 + \sqrt{c(2-c)}\sum_{i=0}^t (1-c)^i \fchosenstep[t-i])^4)$. We neglect in the limit when $t$ goes to $\infty$ the part with $(1-c)^t [\p[0]]_1$, as it converges fast to $0$. So 
\begin{align} \label{eq:devp}
\lim_{t \rightarrow \infty} \E\left(\pdim[1]^4\right) &= \lim_{t \rightarrow \infty} \E\left(c^2(2-c)^2\left(\sum_{i=0}^t(1-c)^i \fchosenstep[t-i]\right)^4 \right) \enspace .
\end{align}
To develop the RHS of Eq.\eqref{eq:devp} we use the following formula: for $(a_i)_{i \in [[1,m]]}$
\begin{align} \label{eq:formula4}
\left(\sum_{i=1}^m a_i\right)^4 = &\sum_{i=1}^m a_i^4 + 4 \sum_{i=1}^m\sum_{\underset{j \neq i}{j=1}}^m a_i^3 a_j + 6 \sum_{i=1}^m \sum_{j=i+1}^m a_i^2 a_j^2 \nonumber\\
 &+ 12 \sum_{i=1}^m \sum_{\underset{j \neq i}{j=1}}^m \sum_{\underset{k \neq i}{k=j+1}}^m a_i^2 a_j a_k + 24 \sum_{i=1}^m \sum_{j=i+1}^m \sum_{k=j+1}^m \sum_{l=k+1}^m a_i a_j a_k a_l \enspace .
\end{align}
This formula will allow us to use the independence over time of $\fchosenstep$ from Lemma~\ref{lem:selectedstep}, so that $\E(\fchosenstep[i]^3 \fchosenstep[j]) = \E(\fchosenstep[i]^3)\E(\fchosenstep[j]) = \E(\Nlambda^3)\E(\Nlambda)$ for $i \neq j$, and so on. We apply Eq~\eqref{eq:formula4} on Eq~\eqref{eq:devp4}, with $a = 1-c$. 
\begin{align} \label{eq:bigdevp4}
\lim_{t \rightarrow \infty} \frac{\E\left(\pdim[1]^4\right)}{c^2(2-c)^2} = &\lim_{t \rightarrow \infty} \sum_{i=0}^t a^{4i} \E\left( \Nlambda^4 \right) + 4 \sum_{i=0}^t\sum_{\underset{j \neq i}{j=0}}^t a^{3i + j} \E\left( \Nlambda^3 \right) \E\left( \Nlambda \right) \nonumber\\ 
&+ 6 \sum_{i=0}^t \sum_{j=i+1}^t a^{2i + 2j} \E\left( \Nlambda^2 \right)^2 \nonumber\\
 &+ 12 \sum_{i=0}^t \sum_{\underset{j \neq i}{j=0}}^t \sum_{\underset{k \neq i}{k=j+1}}^t a^{2i +j + k}\E\left( \Nlambda^2 \right)\E\left( \Nlambda \right)^2 \nonumber \\
 &+ 24 \sum_{i=0}^t \sum_{j=i+1}^t \sum_{k=j+1}^t \sum_{l=k+1}^t a^{i+j+k+l} \E\left( \Nlambda \right)^4 
\end{align}
We now have to develop each term of Eq.~\eqref{eq:bigdevp4}.
\begin{align} \label{eq:k4}
\sum_{i=0}^t a^{4i} &= \frac{1-a^{4(t+1)}}{1-a^4} \nonumber \\
\lim_{t \rightarrow \infty} \sum_{i=0}^t a^{4i} &= \frac{1}{1-a^4}
\end{align}
\begin{equation} \label{eq:k31demi}
\sum_{i=0}^t\sum_{\underset{j \neq i}{j=0}}^t a^{3i + j} = \sum_{i=0}^{t-1}\sum_{j=i+1}^t a^{3i + j} + \sum_{i=1}^t\sum_{j=0}^{i-1} a^{3i + j}
\end{equation}
\begin{align} \label{eq:k31part1}
\sum_{i=0}^{t-1}\sum_{j=i+1}^t a^{3i + j} &= \sum_{i=0}^{t-1} a^{4i+1} \frac{1-a^{t-i}}{1-a} \nonumber \\
\lim_{t \rightarrow \infty} \sum_{i=0}^{t-1}\sum_{j=i+1}^t a^{3i + j} &= \lim_{t \rightarrow \infty} \frac{a}{1-a}\sum_{i=0}^{t-1} a^{4i} \nonumber \\
 &= \frac{a}{(1-a)(1-a^4)}
\end{align}
\begin{align} \label{eq:k31part2}
\sum_{i=1}^t\sum_{j=0}^{i-1} a^{3i + j} &= \sum_{i=1}^t a^{3i} \frac{1-a^{i}}{1-a} \nonumber \\
 &= \frac{1}{1-a} \left( a^3\frac{1-a^{3t}}{1-a^3} - a^4\frac{1-a^{4t}}{1-a^4} \right) \nonumber \\
\lim_{t \rightarrow \infty} \sum_{i=1}^t\sum_{j=0}^{i-1} a^{3i + j} &= \frac{1}{1-a} \left( \frac{a^3}{1-a^3} - \frac{a^4}{1-a^4} \right) \nonumber \\
&= \frac{a^3(1-a^4) - a^4(1-a^3)}{(1-a)(1-a^3)(1-a^4)}  \nonumber \\
&= \frac{a^3 - a^4}{(1-a)(1-a^3)(1-a^4)} 
\end{align}
By combining Eq~\eqref{eq:k31part1} with Eq~\eqref{eq:k31part2} to Eq~\eqref{eq:k31demi} we get
\begin{align} \label{eq:k31}
\lim_{t \rightarrow \infty} \sum_{i=0}^t\sum_{\underset{j \neq i}{j=0}}^t a^{3i + j} &= \frac{a(1-a^3) + a^3 - a^4}{(1-a)(1-a^3)(1-a^4)} = \frac{a(1+a^2-2a^3)}{(1-a)(1-a^3)(1-a^4)} \nonumber \\
&= \frac{a(1-a)(1+a+2a^2))}{(1-a)(1-a^3)(1-a^4)} = \frac{a(1+a+2a^2))}{(1-a^3)(1-a^4)}
\end{align}
\begin{align} \label{eq:k22}
\sum_{i=0}^{t-1} \sum_{j=i+1}^t a^{2i + 2j} &= \sum_{i=0}^{t-1} a^{4i+2} \frac{1-a^{2(t-i)}}{1-a^2} \nonumber \\
\lim_{t \rightarrow \infty} \sum_{i=0}^{t-1} \sum_{j=i+1}^t a^{2i + 2j} &= \frac{a^2}{1-a^2}\sum_{i=0}^{t-1} a^{4i} \nonumber \\
&= \frac{a^2}{(1-a^2)(1-a^4)}
\end{align}
\begin{align} \label{eq:k211tiers}
\sum_{i=0}^t \sum_{\underset{j \neq i}{j=0}}^{t-1} \sum_{\underset{k \neq i}{k=j+1}}^t a^{2i +j + k} = &\sum_{i=2}^t \sum_{j=0}^{i-2} \sum_{k=j+1}^{i-1} a^{2i +j + k} + \sum_{i=1}^{t-1} \sum_{j=0}^{i-1} \sum_{k=i+1}^t a^{2i +j + k} \nonumber\\
&+ \sum_{i=0}^{t-2} \sum_{j=i+1}^{t-1} \sum_{k=j+1}^t a^{2i +j + k}
\end{align}
\begin{align} \label{eq:k211part1}
\sum_{i=2}^t \sum_{j=0}^{i-2} \sum_{k=j+1}^{i-1} a^{2i +j + k} &= \sum_{i=2}^t \sum_{j=0}^{i-2} a^{2i+2j+1} \frac{1-a^{i-j-1}}{1-a} \nonumber \\
&= \frac{1}{1-a}\sum_{i=2}^t a^{2i+1}\frac{1-a^{2(i-1)}}{1-a^2} - a^{3i}\frac{1-a^{i-1}}{1-a} \nonumber \\
&= \frac{1}{1-a} \bigg( \frac{a^5}{1-a^2} \frac{1-a^{2(t-1)}}{1-a^2} - \frac{a^7}{(1-a^2)} \frac{1-a^{4(t-1)}}{1-a^4} \nonumber \\
 &~~~~~~- \frac{a^6}{1-a}\frac{1-a^{3(t+1)}}{1-a^3} + \frac{a^7}{1-a}\frac{1-a^{4(t+1)}}{1-a^4} \bigg) \nonumber \\
&\underset{t \rightarrow \infty}{\longrightarrow} \frac{a^5}{1-a} \bigg( \frac{1}{(1-a^2)^2} -\frac{a^2}{(1-a^2)(1-a^4)} \nonumber \\
&~~~~~~~~- \frac{a}{(1-a)(1-a^3)} + \frac{a^2(1+a)}{(1+a)(1-a)(1-a^4)}\bigg) \nonumber \\
&\underset{t \rightarrow \infty}{\longrightarrow} \frac{a^5}{1-a} \bigg( \frac{(1+a^2)}{(1-a^2)^2(1+a^2)} + \frac{a^3}{(1-a^2)(1-a^4)} \nonumber \\
&~~~~~~~~- \frac{a}{(1-a)(1-a^3)} \bigg) \nonumber \\
&\underset{t \rightarrow \infty}{\longrightarrow} \frac{a^5}{1-a} \bigg(\frac{1+a^2+a^3}{(1+a)(1-a)(1-a^4)} - \frac{a}{(1-a)(1-a^3)} \bigg) \nonumber \\
&\underset{t \rightarrow \infty}{\longrightarrow} \frac{a^5}{(1-a)^2}\frac{(1+a^2+a^3)(1-a^3) - a(1+a)(1-a^4)))}{(1+a)(1-a^3)(1-a^4)} \nonumber \\
&\underset{t \rightarrow \infty}{\longrightarrow} a^5\frac{1+a^2-a^5-a^6 -(a+a^2-a^5-a^6)}{(1-a)(1-a^2)(1-a^3)(1-a^4)} \nonumber \\
&\underset{t \rightarrow \infty}{\longrightarrow} \frac{a^5}{(1-a^2)(1-a^3)(1-a^4)}
\end{align}
\begin{align} \label{eq:k211part2}
\sum_{i=1}^{t-1} \sum_{j=0}^{i-1} \sum_{k=i+1}^t a^{2i +j + k} &= \sum_{i=1}^{t-1} \sum_{j=0}^{i-1} a^{3i +j + 1}\frac{1-a^{t-i}}{1-a} \nonumber \\
& \underset{t \rightarrow \infty}{\longrightarrow} \lim_{t \rightarrow \infty} \frac{a}{1-a}\sum_{i=1}^{t-1} a^{3i} \frac{1-a^{i}}{1-a} \nonumber \\
& \underset{t \rightarrow \infty}{\longrightarrow} \lim_{t \rightarrow \infty} \frac{a}{(1-a)^2} \left( a^3\frac{1-a^{3t}}{1-a^3} - a^4\frac{1-a^{4t}}{1-a^4} \right) \nonumber \\
& \underset{t \rightarrow \infty}{\longrightarrow} \frac{a}{(1-a)^2} \left( \frac{a^3(1-a^4) - a^4(1-a^3)}{(1-a^3)(1-a^4)} \right) \nonumber \\
& \underset{t \rightarrow \infty}{\longrightarrow} \frac{a^4 - a^5}{(1-a)^2(1-a^3)(1-a^4)} =  \frac{a^4}{(1-a)(1-a^3)(1-a^4)}
\end{align}
\begin{align} \label{eq:k211part3}
\sum_{i=0}^{t-2} \sum_{j=i+1}^{t-1} \sum_{k=j+1}^t a^{2i +j + k} &= \sum_{i=0}^{t-2} \sum_{j=i+1}^{t-1} a^{2i+2j+1}\frac{1-a^{t-j}}{1-a} \nonumber \\
& \underset{t \rightarrow \infty}{\longrightarrow} \lim_{t \rightarrow \infty} \frac{a}{1-a} \sum_{i=0}^{t-2} a^{4i+2} \frac{1-a^{2(t-i-1)}}{1-a^2} \nonumber \\
& \underset{t \rightarrow \infty}{\longrightarrow} \lim_{t \rightarrow \infty} \frac{a^3}{(1-a)(1-a^2)} \frac{1-a^{4(t-1)}}{1-a^4} \nonumber \\
 &\underset{t \rightarrow \infty}{\longrightarrow} \frac{a^3}{(1-a)(1-a^2)(1-a^4)}
\end{align}
We now combine Eq~\eqref{eq:k211part1}, Eq.~\eqref{eq:k211part2} and Eq.~\eqref{eq:k211part1} in Eq.~\eqref{eq:k211tiers}.
\begin{align} \label{eq:k211}
\sum_{i=0}^t \sum_{\underset{j \neq i}{j=0}}^{t-1} \sum_{\underset{k \neq i}{k=j+1}}^t a^{2i +j + k} &\underset{t \rightarrow \infty}{\longrightarrow} \frac{a^5(1-a) + a^4(1-a^2) + a^3(1-a^3)}{(1-a)(1-a^2)(1-a^3)(1-a^4)} \nonumber \\
 &\underset{t \rightarrow \infty}{\longrightarrow} \frac{a^3 + a^4 + a^5 - 3a^6}{(1-a)(1-a^2)(1-a^3)(1-a^4)} \nonumber \\
 &\underset{t \rightarrow \infty}{\longrightarrow} \frac{a^3(1+ 2a + 3a^2)}{((1-a^2)(1-a^3)(1-a^4)} \nonumber \\
\end{align}
\begin{align} \label{eq:k1111}
\sum_{i=0}^{t-3} \sum_{j=i+1}^{t-2} \sum_{k=j+1}^{t-1} \sum_{l=k+1}^t a^{i+j+k+l} &= \sum_{i=0}^{t-3} \sum_{j=i+1}^{t-2} \sum_{k=j+1}^{t-1} a^{i+j+2k+1} \frac{1-a^{t-k}}{1-a} \nonumber \\
& \underset{t \rightarrow \infty}{\longrightarrow} \lim_{t \rightarrow \infty} \frac{a}{1-a}  \sum_{i=0}^{t-3} \sum_{j=i+1}^{t-2} a^{i+3j+2} \frac{1-a^{2(t-1-j)}}{1-a^2} \nonumber \\
& \underset{t \rightarrow \infty}{\longrightarrow} \lim_{t \rightarrow \infty} \frac{a^3}{(1-a)(1-a^2)}  \sum_{i=0}^{t-3} a^{4i+3} \frac{1-a^{3(t-2-i)}}{1-a^3} \nonumber \\
& \underset{t \rightarrow \infty}{\longrightarrow} \lim_{t \rightarrow \infty} \frac{a^6}{(1-a)(1-a^2)(1-a^3)}  \frac{1-a^{4(t-2)}}{1-a^4} \nonumber \\
& \underset{t \rightarrow \infty}{\longrightarrow} \frac{a^6}{(1-a)(1-a^2)(1-a^3)(1-a^4)}
\end{align}
By factorising Eq.~\eqref{eq:k4}, Eq.~\eqref{eq:k31}, Eq.~\eqref{eq:k22}, Eq.~\eqref{eq:k211} and Eq.~\eqref{eq:k1111} by $\frac{1}{1-a^4}$ we get the coefficients of Theorem~\ref{th:var}.
\qed\end{proof}

\begin{figure}[t,b]
\centering\vspace*{-1ex}
\subfloat[Without cumulation ($c=1$)]{\label{fig:simwithoutcumulation}\includegraphics[width=0.5\textwidth,trim=0 0 0 12mm,clip]{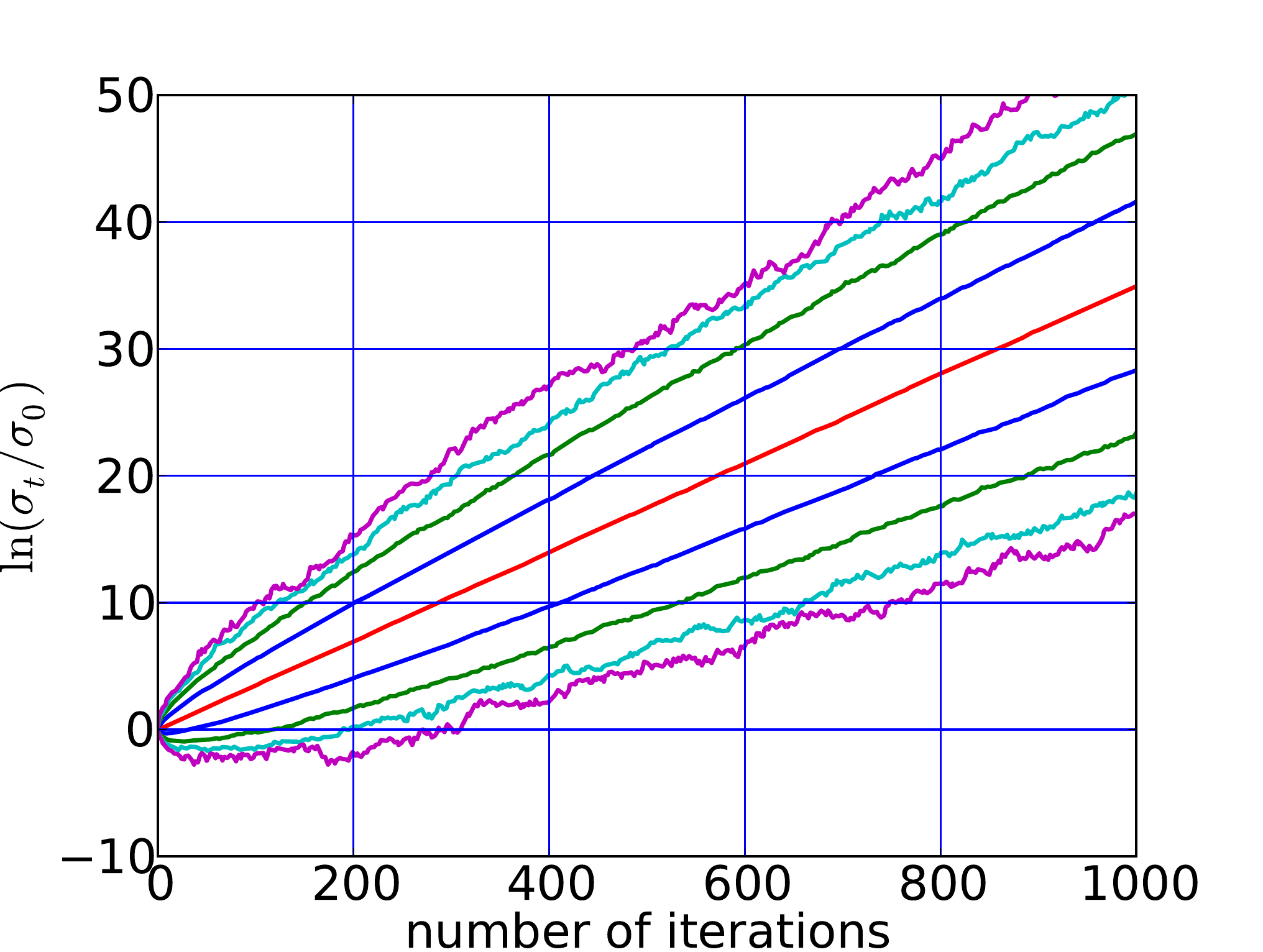}}
\subfloat[With cumulation ($c = 1/\sqrt{20}$)]{\label{fig:simwithcumulation}\includegraphics[width=0.5\textwidth,trim=0 0 0 12mm,clip]{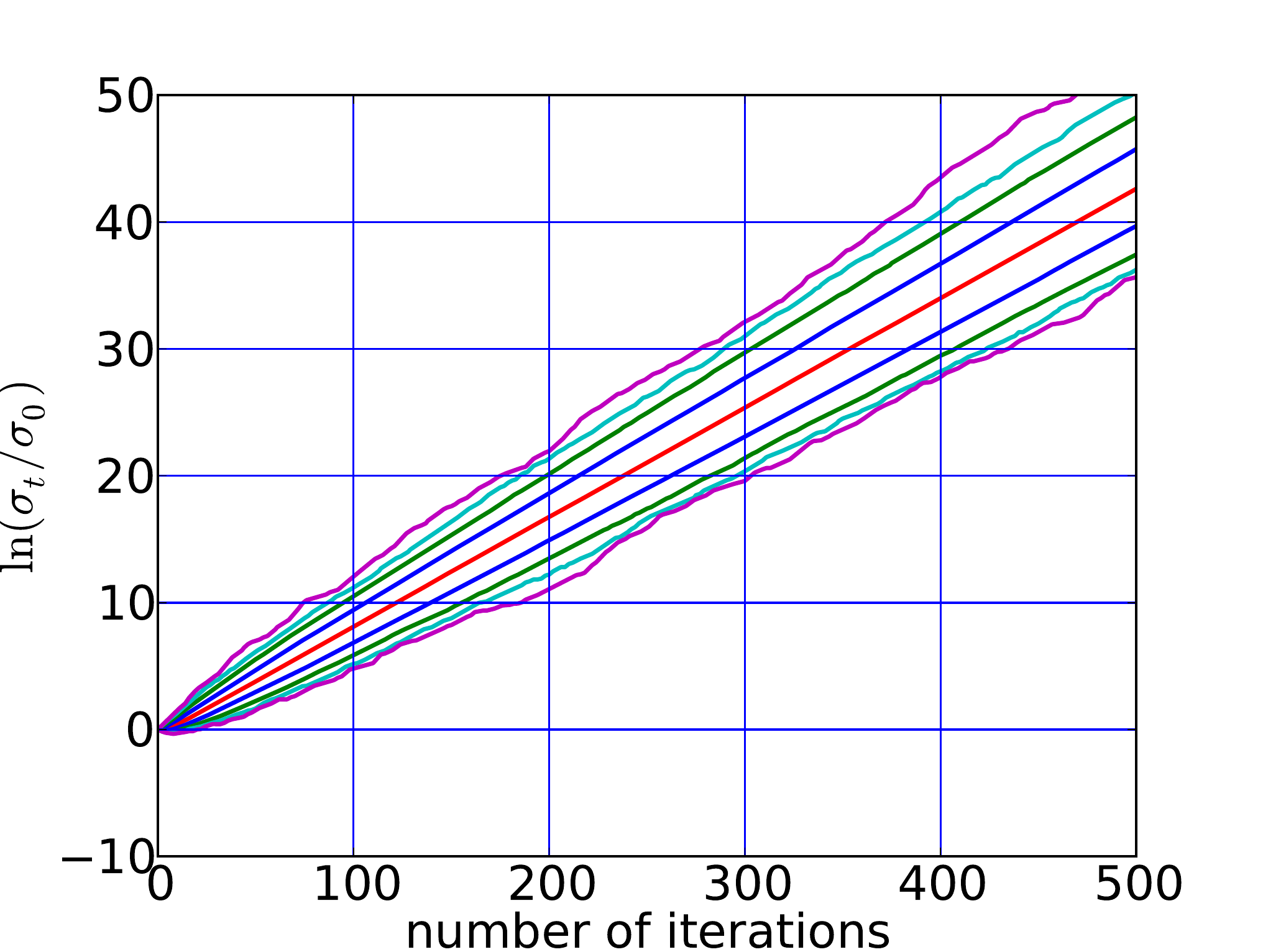}}
\caption{\label{fig:logsigmaevolution}
$\ln(\sig / \sigma_0)$ against $t$. The different curves represent the quantiles of a set of $5.10^3+1$ samples, more precisely the $10^{i}$-quantile and the $1 - 10^{-i}$-quantile for $i$ from $1$ to $4$; and the median. We have $n=20$ and $\lambda = 8$. }
\end{figure}

Figure~\ref{fig:logsigmaevolution} shows the time evolution of $\ln(\sigma_t / \sigma_0)$ for 5001 runs and $c=1$ (left) and $c=1/\sqrt{n}$ (right).
By comparing Figure~\ref{fig:simwithoutcumulation} and Figure~\ref{fig:simwithcumulation} we observe smaller variations of $\ln(\sigma_t / \sigma_0)$ with the smaller value of $c$.

\begin{figure}
\includegraphics[width=0.49\textwidth]{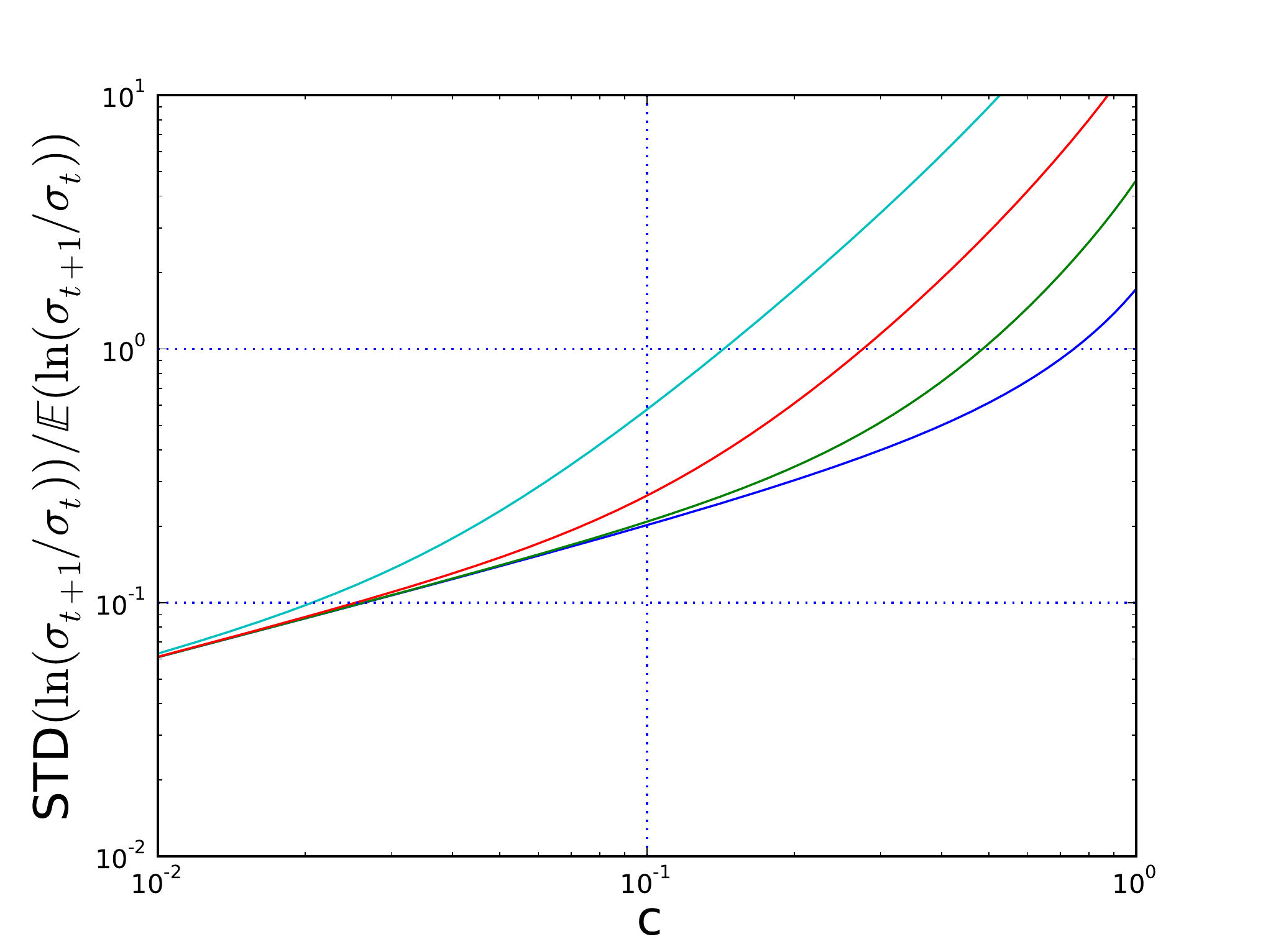}
\includegraphics[width=0.49\textwidth]{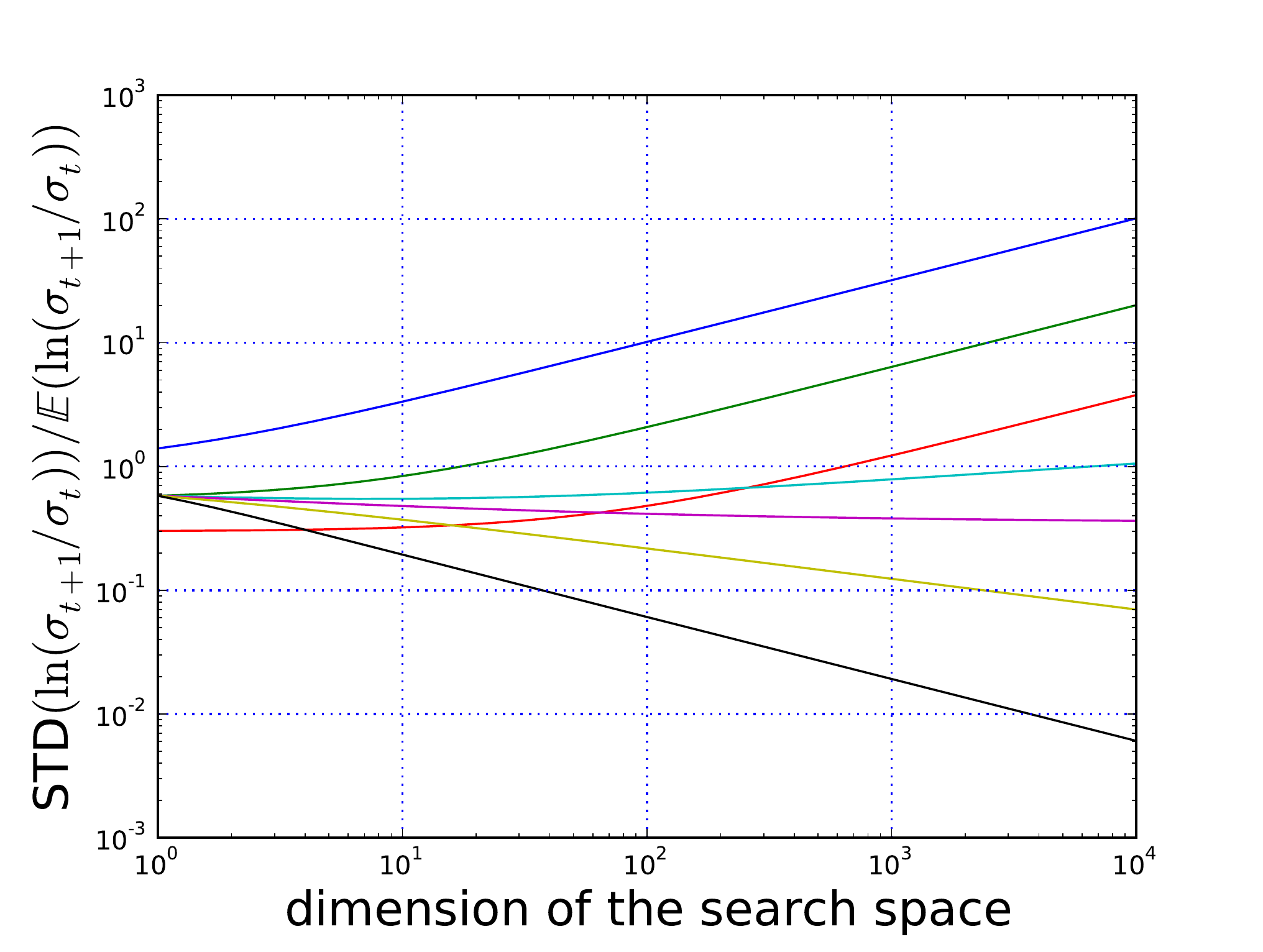}
\caption{\label{fig:stdandc}
Standard deviation of $\ln\left(\sigma_{t+1}/\sig\right)$ relatively to its expectation. Here $\lambda = 8$. The curves were plotted using Eq.~\eqref{eq:var} and Eq.~\eqref{eq:devp4}. On the left, curves for (right to left) $n=2$, $20$, $200$ and $2000$. On the right, different curves for (top to bottom) $c = 1$, $0.5$, $0.2$, $1/(1+n^{1/4})$, $1/(1+n^{1/3})$, $1/(1+n^{1/2})$ and $1/(1+n)$.}
\end{figure}

Figure~\ref{fig:stdandc} shows the relative standard deviation of $\ln\left(\sigma_{t+1}/ \sigma_t \right)$ (i.e. the standard deviation divided by its expected value). 
Lowering $c$, as shown in the left, decreases the relative standard deviation. To get a value below one, $c$ must be smaller for larger dimension. In agreement with Theorem \ref{th:var}, In Figure~\ref{fig:stdandc}, right, the relative standard deviation increases like $\sqrt{n}$ with the dimension for constant $c$ (three increasing curves). A careful study \cite{chotard2012TRcumulative} of the variance equation of Theorem \ref{th:var} shows that for the choice of $c = 1/(1+n^{\alpha})$, if $\alpha > 1/3$ the relative standard deviation converges to $0$ with $\sqrt{(n^{2\alpha} + n)/n^{3 \alpha}}$. Taking $\alpha = 1/3$ is  a critical value where the relative standard deviation converges to $1/(\sqrt{2} \E(\Nlambda)^2)$. On the other hand, lower values of $\alpha$ makes the relative standard deviation diverge with $n^{(1-3\alpha)/2}$.

\section{Summary} \label{sec:conclusion}

We investigate throughout this paper the ($1,\lambda$)-CSA-ES on affine linear functions composed with strictly increasing transformations. We find, in Theorem~\ref{th:geometricdivergencecumul}, the limit distribution for $\ln(\sig/\sigma_0)/t$ and rigorously prove the desired behaviour of $\sigma$ with $\lambda\ge3$~ for any $c$, and with $\lambda=2$ and cumulation ($0<c<1$): the step-size diverges geometrically fast. In contrast, without cumulation ($c=1$) and with $\lambda=2$, a random walk on $\ln (\sigma)$ occurs, like for the ($1,2$)-$\sigma$SA-ES \cite{hansen2006ecj} (and also for the same symmetry reason). We derive an expression for the variance of the step-size increment. On linear functions when $c = 1/n^{\alpha}$, for $\alpha \geq 0$ ($\alpha = 0$ meaning $c$ constant) and for $n \to \infty$ the standard deviation is about $\sqrt{(n^{2\alpha} + n)/n^{3 \alpha}}$ times larger than the step-size increment. From this follows that keeping $c < 1/n^{1/3}$ ensures that the standard deviation of $\ln(\sigma_{t+1}/ \sigma_t)$ becomes negligible compared to $\ln(\sigma_{t+1}/ \sigma_t)$ when the dimensions goes to infinity. That means, the signal to noise ratio goes to zero, giving the algorithm strong stability. The result confirms that even the largest default cumulation parameter $c=1/\sqrt{n}$ is a stable choice.

\section*{Acknowledgments}
This work was partially supported by the ANR-2010-COSI-002
grant (SIMINOLE) of the French National Research Agency and the ANR COSINUS project ANR-08-COSI-007-12.

\bibliography{biblio}

\bibliographystyle{plain}

\end{document}